\documentclass{article}

\PassOptionsToPackage{numbers,sort}{natbib}

\usepackage[main, final]{neurips_2026}

\usepackage[utf8]{inputenc} 
\usepackage[T1]{fontenc}    
\usepackage{hyperref}       
\usepackage{url}            
\usepackage{booktabs}       
\usepackage{amsfonts}       
\usepackage{nicefrac}       
\usepackage{microtype}      

\usepackage{enumitem}
\usepackage{float}

\usepackage{amssymb,amsthm,mathtools,bm}
\usepackage{multirow}
\usepackage{wrapfig}
\usepackage[table]{xcolor}

\theoremstyle{plain}
\newtheorem{theorem}{Theorem}
\newtheorem{proposition}{Proposition}
\newtheorem{lemma}{Lemma}
\newtheorem{corollary}{Corollary}
\theoremstyle{definition}
\newtheorem{definition}{Definition}
\newtheorem{assumption}{Assumption}
\theoremstyle{remark}

\usepackage{array}
\newcolumntype{L}[1]{>{\raggedright\let\newline \\\arraybackslash\hspace{0pt}}m{#1}}
\newcolumntype{C}[1]{>{\centering\let\newline \\\arraybackslash\hspace{0pt}}m{#1}}
\newcolumntype{R}[1]{>{\raggedleft\let\newline \\\arraybackslash\hspace{0pt}}m{#1}}

\title{Attention Sinks Induce Gradient Sinks: Massive Activations as Gradient Regulators in Transformers}
\author{
    Yihong Chen$^1$
    \qquad Zhouchen Lin$^{2,3}$\thanks{Co-corresponding authors.}
    \qquad Quanming Yao$^{1,4,5}$\footnotemark[1] \\
    $^1$Department of Electronic Engineering, Tsinghua University \\ 
    $^2$State Key Lab of General Artificial Intelligence \\ 
    $^3$School of Intelligence Science and Technology, Peking University\\
    $^4$State Key Laboratory of Space Network and Communications \\
    $^5$Beijing National Research Center for Information Science and Technology \\
    \texttt{chenyihong@tsinghua.edu.cn}, 
    \quad \texttt{zlin@pku.edu.cn},
    \quad \texttt{qyaoaa@tsinghua.edu.cn}
}

\begin{document}
\maketitle

\begin{abstract}
Attention sinks and massive activations are recurring and closely related phenomena in Transformer models. 
Existing explanations have largely focused on the forward pass, yet in pre-norm Transformers, large residual-stream norms play only an indirect forward role because sublayers operate on normalized inputs.
We study this relationship from the perspective of backpropagation. Empirically and theoretically, we show that under causal masking, attention sinks can induce pronounced gradient concentration, which we term \emph{gradient sinks}. Since the RMSNorm Jacobian attenuates gradients roughly in inverse proportion to input norm, massive activations can be understood as adaptive regulators of this localized gradient pressure during training.
This interpretation predicts that attenuating sink-induced gradients should weaken massive activations. We test this prediction with V-scale, a modification that adjusts backpropagated gradients on the value paths. In V-scale models and a backward-only control, attention sinks are preserved, whereas massive activations are suppressed. These results identify gradient sinks as a backward-pass counterpart of attention sinks, and massive activations as an adaptive RMSNorm-mediated response that attenuates the resulting localized training pressure. Our code is available at \href{https://github.com/LARS-research/GradientSink}{https://github.com/LARS-research/GradientSink}.
\end{abstract}

\section{Introduction}

Two prominent and frequently co-occurring phenomena in Transformer-based large language models (LLMs) are \emph{attention sinks} (AS), where a small number of tokens attract disproportionate attention mass~\cite{xiao2024streamingllm,gu2025when}, and \emph{massive activations} (MA), where activations become unusually large on a small set of tokens and features~\cite{bondarenko2023quantizable,sun2024massive,an2025systematic}. 
These phenomena are predominantly observed at the first token and are of significant practical importance for long-context inference~\cite{xiao2024streamingllm,su2025kvsink}, fine-tuning~\cite{liu2026sinktrack,liu2025all,fu2026attention,liu2026surgery}, and quantization~\cite{xiao2023smoothquant,bondarenko2023quantizable}. The frequent co-occurrence of AS and MA has also attracted mechanistic interest. Existing work has provided several forward-pass explanations for why AS and MA may arise or persist. For example, AS has commonly been explained as a way to implement ``no-op'' attention heads that contribute little to the current token~\cite{gu2025when,barbero2025why,guo2025activedormant} under the sum-to-one constraint of the softmax operator, while MA has been described as helping form attention-sink behavior~\cite{sun2024massive,su2026unveiling,queipo-de-llano2026attention}.

However, in modern pre-norm architectures, attention and MLP sublayers operate on normalized inputs~\cite{zhang2019rmsnorm,xiong2020layernorm,touvron2023llama}, making local forward attention depend primarily on the direction of representation rather than on its raw norm. This makes the emergence of MA appear difficult to explain from a purely forward-pass perspective. Moreover, many post-hoc interventions in the literature modify MLP outputs or token activations of trained models. While informative, such interventions often perturb both magnitude and direction, making it difficult to isolate the causal mechanism linking AS and MA. Recent studies of training dynamics~\cite{gallegofeliciano2026hiddendynamics,queipo-de-llano2026attention} and architectural decoupling~\cite{sun2026anatomy} clarify their development and separability, but a training-time puzzle remains: \emph{why do standard pre-norm Transformers learn such extreme activations at sink tokens?}

We approach this question from the backward pass. Residual-stream scale can be important during backpropagation because it controls how much gradient signal is transmitted through normalization in pre-norm architectures. This matters precisely at sink tokens. In an attention head under causal masking, when many later tokens attend to a sink token, their training signals are routed back to that token as well. This turns a forward attention sink into a backward concentration of gradients, which we call a gradient sink (GS).

In this view, massive activations can be understood as learned regulators that absorb sink-induced gradient pressure and help maintain stable residual-stream gradient transport. This interpretation further predicts that relieving such pressure through an alternative attenuation route should reduce reliance on MA even while AS remains present. The remainder of the paper develops this mechanism from observation and analysis to intervention.

Our contributions are threefold:
\begin{itemize}
\item \textbf{Gradient sinks.} We identify gradient sinks as a backward-pass counterpart of attention sinks in pre-norm Transformers. We show that concentrated gradient pressure at the first token is widespread across models.
\item \textbf{Mechanism.} We provide an empirical and theoretical account of massive activations as adaptive regulators. In pre-norm architectures, RMSNorm makes activation scale an effective local factor to attenuate gradient pressure induced by attention sinks, thereby stabilizing gradient transport across the residual stream.
\item \textbf{Intervention.} We propose \emph{V-scale}, a value-path gradient valve that attenuates sink-induced gradients. V-scale and its backward-only control substantially suppress massive activations while retaining attention sinks, supporting the proposed mechanism. We also observe improvements on multi-key retrieval across several context lengths and precision settings.
\end{itemize}

\section{Setup}
\label{sec:setup}

We introduce the notation needed for the empirical measurements and theoretical analysis that follow.
In this paper, token positions are indexed by $t\in\{0,\dots,T-1\}$, and $\mathcal L$ denotes the training loss.

\paragraph{Pre-norm Transformer and attention}
We study decoder-only Llama-like Transformers with pre-norm residual blocks, RMSNorm, RoPE positional encoding, and SwiGLU MLPs~\cite{touvron2023llama,zhang2019rmsnorm,su2021rope,shazeer2020gluvariants}. Figure~\ref{fig:architecture_gradient}~(a) illustrates the forward computation.
Let $H^\ell = (h_0^\ell,\dots,h_{T-1}^\ell)^\top \in \mathbb{R}^{T\times d_{\mathrm{model}}}$ denote the hidden states input to layer $\ell$.
A pre-norm block updates the sequence as
\[\begin{aligned}
    \widetilde H^\ell &= \mathrm{RMSNorm}(H^\ell), &
    R^{\mathrm{attn},\ell} &= \mathrm{Attention}(\widetilde H^\ell), &
    H^{\ell+\nicefrac12} &= H^\ell + R^{\mathrm{attn},\ell}, \\
    \widetilde H^{\ell+\nicefrac12} &= \mathrm{RMSNorm}(H^{\ell+\nicefrac12}), &
    R^{\mathrm{mlp},\ell} &= \mathrm{MLP}(\widetilde H^{\ell+\nicefrac12}), &
    H^{\ell+1} &= H^{\ell+\nicefrac12} + R^{\mathrm{mlp},\ell}.
\end{aligned}\]
We write $r_t^{\mathrm{attn},\ell}$ and $r_t^{\mathrm{mlp},\ell}$ for the $t$-th rows of $R^{\mathrm{attn},\ell}$ and $R^{\mathrm{mlp},\ell}$.

\begin{figure}[tb]
    \centering
    \includegraphics[width=\linewidth]{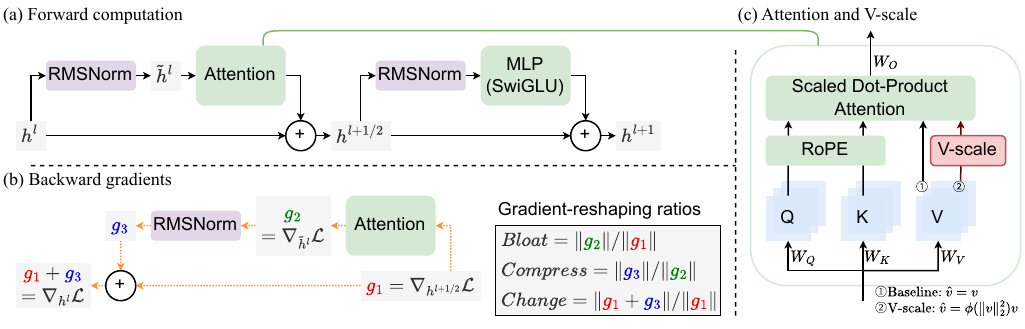}
    
    \caption{
    Pre-norm Transformer architecture, gradient measurements, and V-scale placement.
    (a) Forward computation through the attention and MLP sublayers.
    (b) Backward flow and gradient measurements for the attention sublayer. Detailed definitions are given in Section~\ref{sec:ma_grad_reshape}.
    (c) Expanded attention computation. Value paths 1 and 2 are alternatives for the baseline model and the V-scale variant proposed in Section~\ref{sec:vscale}, respectively.
    }\label{fig:architecture_gradient}
\end{figure}

The attention block can be either multi-head attention (MHA)~\cite{vaswani2017attention} or a variant such as grouped-query attention (GQA)~\cite{ainslie2023gqa}. For simplicity, we fix one layer $\ell$ and one attention head $h$, and suppress the layer/head superscripts unless needed. With RoPE, the query, key, and value states at token position $t$ are $q_t = R_t W_Q \widetilde h_t$, $k_t = R_t W_K \widetilde h_t$, and $v_t = W_V \widetilde h_t$, where $R_t\in\mathbb{R}^{d_{\mathrm{head}}\times d_{\mathrm{head}}}$ is the RoPE rotation matrix. The causal attention logits and weights are computed by $z_{tj}=\langle q_t,k_j\rangle/\sqrt{d_{\mathrm{head}}} + m_{tj}$ and $a_{tj}=\mathrm{softmax}(z_{t,:})_j$, where $m_{tj}=0$ for $j\le t$ and $m_{tj}=-\infty$ otherwise. Finally, the value aggregation and head output are $y_t = \sum\nolimits_{j\le t} a_{tj} v_j$ and $o_t = W_O y_t$, respectively.

\paragraph{Attention sink and activation measurements}
For a candidate sink position $s$ (typically the first token, i.e., $s=0$), we follow~\cite{gu2025when,queipo-de-llano2026attention} and use thresholded criteria to count sink-like behavior:
\begin{equation}\label{eq:sink_rate}
    \mathrm{Sink}_s^{\epsilon,\ell} = \frac{1}{N_{\mathrm{head}}} \sum\nolimits_{h=1}^{N_{\mathrm{head}}} \mathbb{I}(\alpha_s^{\ell,h} > \epsilon),
    \quad
    \alpha_s^{\ell,h}=\frac{1}{T_\epsilon}\sum\nolimits_{t=s}^{s+T_\epsilon-1} a_{ts}^{\ell,h},
\end{equation}
where $N_{\mathrm{head}}$ is the number of attention heads in each layer.
Following prior work~\cite{gu2025when} we set the threshold $\epsilon=0.3$ and $T_\epsilon=64$.
For theory and training-dynamics measurements, we also use the attention-column mass and second moment
\begin{equation}\label{eq:sink_mass}
M_s^{\ell,h} = \sum\nolimits_{t=s}^{T-1} a_{ts}^{\ell,h},
\quad
S_s^{\ell,h} = \sum\nolimits_{t=s}^{T-1} (a_{ts}^{\ell,h})^2.
\end{equation}
Thus $M_s$ measures the total attention mass routed through token $s$ by all causally allowed future queries, while $S_s$ is its second-moment analogue.
In experiments, we focus on the first token ($s=0$).

For token-wise activation analyses, we track activation norms at several computational sites, including the input of Transformer layers $\|h_{t,\mathrm{in}}^\ell\|=\|h_t^\ell\|_2$, the hidden states after attention $\|h_{t,\mathrm{half}}^\ell\|=\|h_t^{\ell+\nicefrac12}\|_2$, and the output of Transformer layers $\|h_{t,\mathrm{out}}^\ell\|=\|h_t^{\ell+1}\|_2$.
We also measure branch outputs $\|r_t^{\mathrm{attn},\ell}\|_2$ and $\|r_t^{\mathrm{mlp},\ell}\|_2$, together with value-state norms $\|v_t^\ell\|_2$, when needed.

\section{Empirical Evidence}
\label{sec:as_to_gs}

This section establishes the empirical basis for viewing massive activations as gradient regulators.
We analyze typical open-source pretrained LLMs~\cite{qwen2.5,qwen3,grattafiori2024llama3herdmodels,glm2024chatglm}, together with \texttt{LlamaForCausalLM} baselines that we pretrain from scratch on the C4~\cite{raffel2020exploring} dataset at scales ranging from 0.1B to 1B parameters. Detailed model and training configurations are deferred to Appendix~\ref{app:config}.

\subsection{From attention sinks to gradient sinks}\label{sec:gs}

To trace how sink structure reappears in backpropagation, we first define gradient-sink behavior.

\begin{definition}[Gradient sink]\label{def:gradient_sink}
Fix a layer $\ell$, a candidate sink position $s$, and an early-token reference set $\{0,\dots,T_e\}$ where $0\le s\le T_e$.
Let $\|\nabla_{\widetilde h_t^\ell}\mathcal L\|_2$ denote the token-wise gradient norm at the post-RMSNorm attention input. The gradient-sink ratio of token $s$ at layer $\ell$, relative to $\{0,\dots,T_e\}$, is defined as
\begin{equation}\label{eq:gs_def}
R_{GS}^{\ell}(s;T_e) :=
\|\nabla_{\widetilde h_s^{\ell}} \mathcal{L}\|_2 /
\big( \nicefrac{1}{T_e} \sum\nolimits_{0 \le t \le T_e, t\ne s} \|\nabla_{\widetilde h_t^{\ell}} \mathcal{L}\|_2 \big).
\end{equation}
For a prescribed threshold $\tau>1$, token $s$ is called a $\tau$-gradient sink at layer $\ell$ if $R_{GS}^{\ell}(s;T_e)\ge \tau$.
\end{definition}

\begin{figure}[t]
	\centering
	\includegraphics[width=0.48\textwidth]{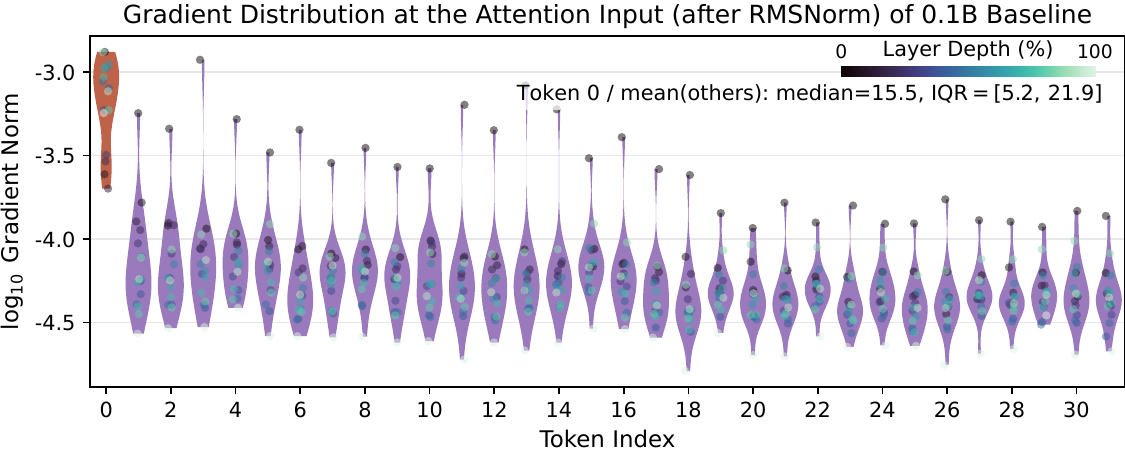}
	\includegraphics[width=0.48\textwidth]{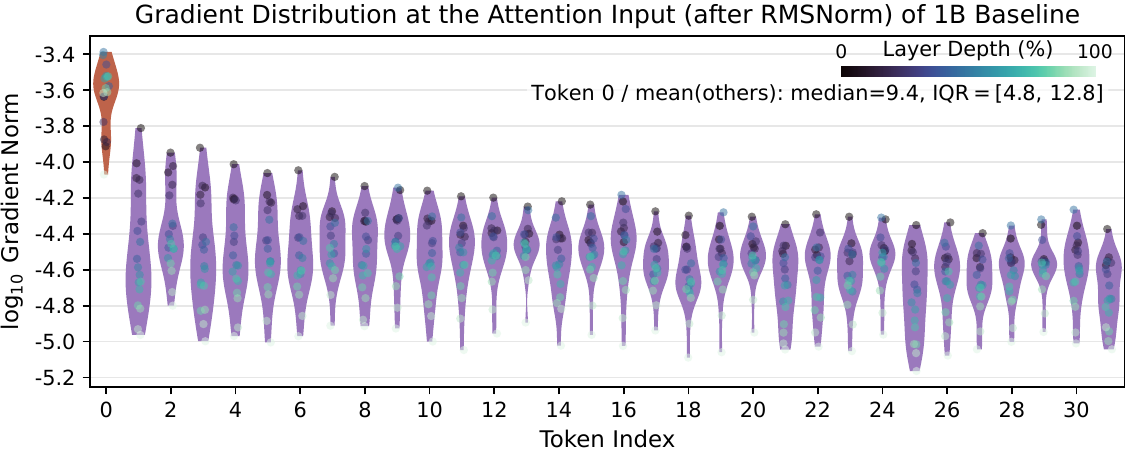}
	\includegraphics[width=0.48\textwidth]{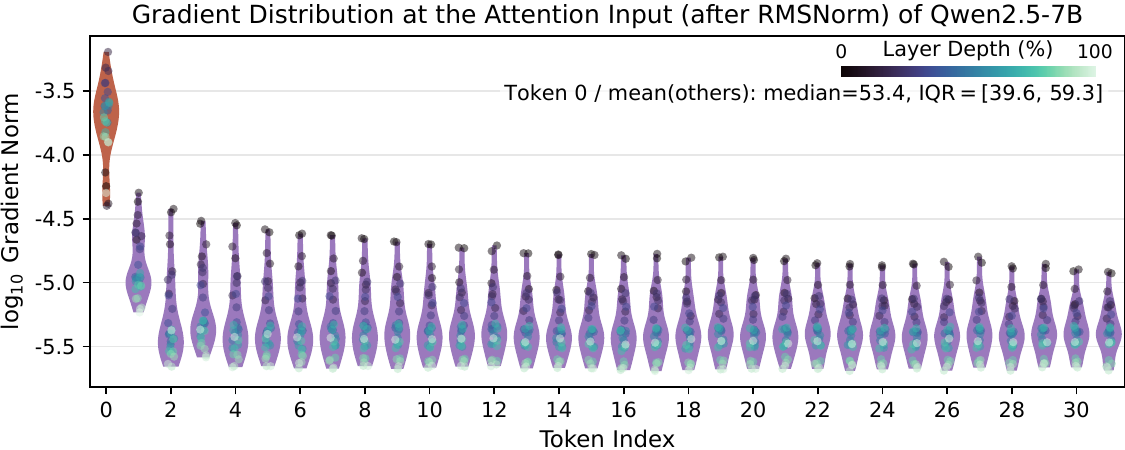}
	\includegraphics[width=0.48\textwidth]{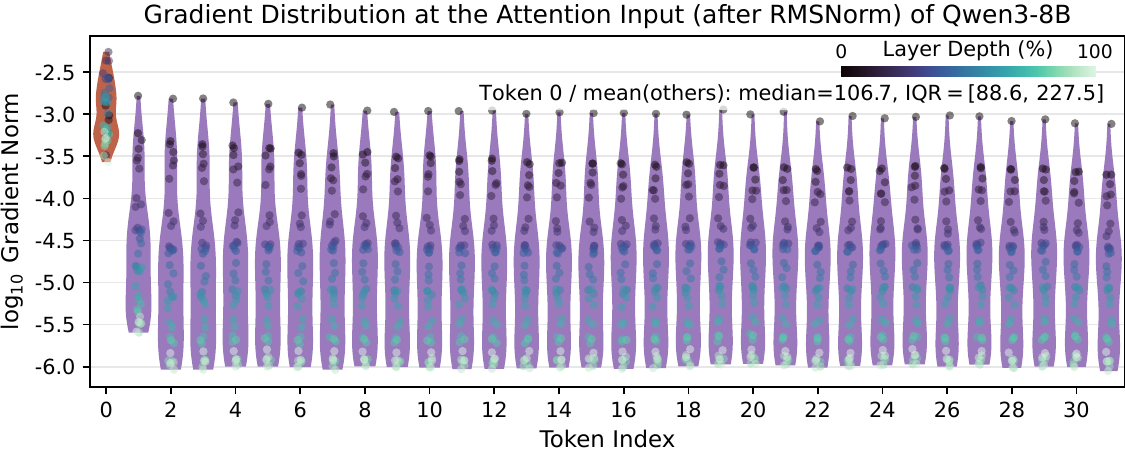}
	\includegraphics[width=0.48\textwidth]{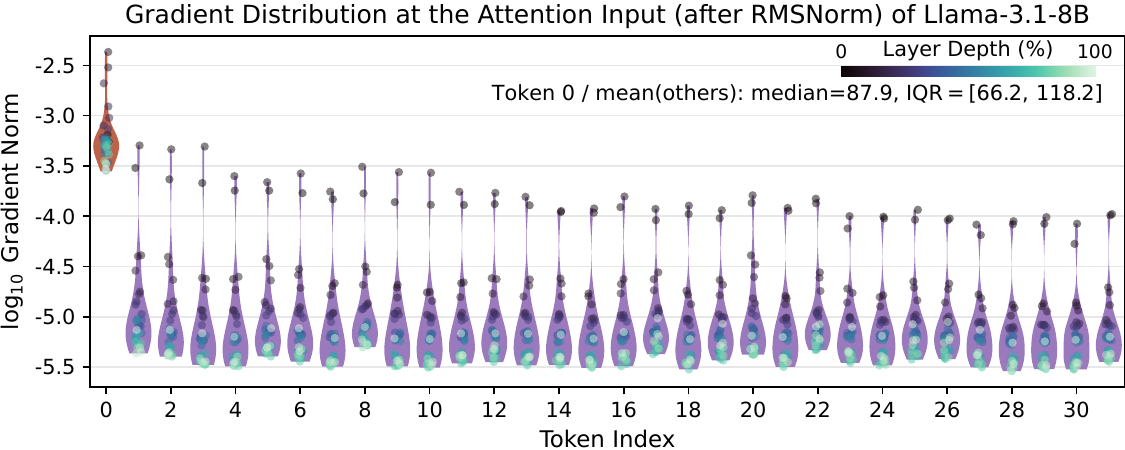}
	\includegraphics[width=0.48\textwidth]{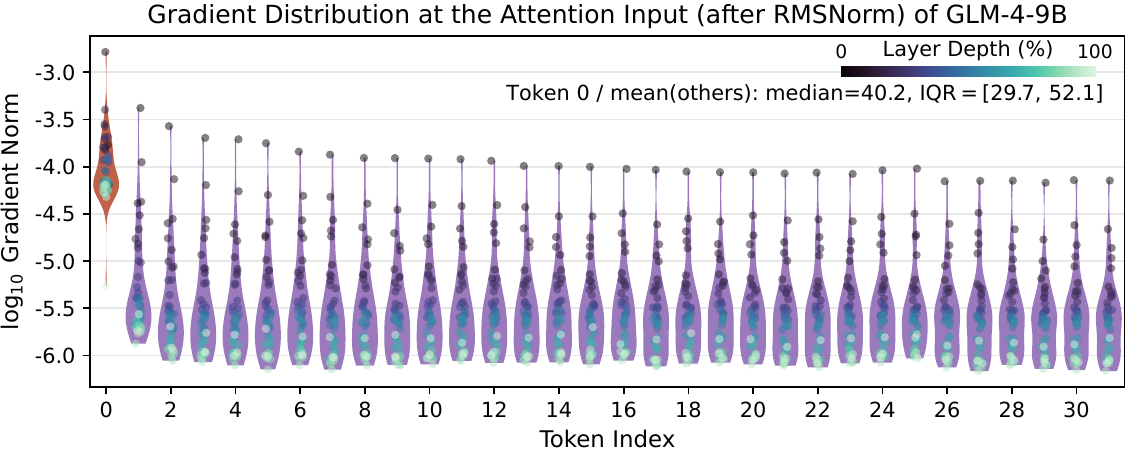}
	
	\caption{
		Gradient sinks across models. Each panel plots the token-wise distribution of $\log_{10}\|\nabla_{\widetilde h_t^{\ell}}\mathcal L\|_2$ at the post-RMSNorm attention input. Token 0 is highlighted. The reported median and interquartile range (IQR) of $R_{GS}^{\ell}(0)$ summarize how much larger the gradient on token 0 is than the mean over positions 1--31. Across both scratch-trained baselines and pretrained models, the first token consistently exhibits the strongest gradient concentration. Models include 0.1B and 1B baselines together with Qwen2.5-7B, Qwen3-8B, Llama-3.1-8B, and GLM-4-9B.}
	\label{fig:gs-definition}
\end{figure}

In this work, we primarily report the continuous ratio $R_{GS}^{\ell}(s;T_e)$ rather than fixing a universal threshold $\tau$, since the meaningful scale of gradient concentration can depend on model size, layer, and measurement protocol. In the experiments below, we take $s=0$ and $T_e=31$.
Gradients are first aggregated over the full batch as in one optimization step and then converted to token-wise $\ell_2$ norms, making the comparison reflect the effective training signal seen by the optimizer.
For scratch-trained baselines, we use the corresponding batch size and context length specified in Appendix~\ref{app:config}; for pretrained models, we use a diagnostic batch of 512 packed sequences with context length 2048. Note that for pretrained models, these gradients should be interpreted as post-hoc diagnostic probes rather than the exact gradients from their official training runs.

Figure~\ref{fig:gs-definition} shows that the sink token consistently exhibits substantially larger gradients than other early tokens, with the reported medians of $R_{GS}^{\ell}(0)$ typically lying in the $10$--$100$ range, i.e., roughly one to two orders of magnitude above the early-token average.
Gradient sinks therefore appear as a common empirical pattern during the training of Transformer language models.

To understand how this concentration is structured inside the attention block, we further decompose the gradients into query, key, and value pathways and track their token-wise norms across training checkpoints. For each token $s$, the incoming gradient splits as $\nabla_{\widetilde h_s^\ell}\mathcal{L} = W_Q^\top \nabla_{q_s^\ell}\mathcal{L} + W_K^\top \nabla_{k_s^\ell}\mathcal{L} + W_V^\top \nabla_{v_s^\ell}\mathcal{L}$, so the observed token-wise concentration can be examined pathway by pathway.

\begin{figure}[tb]
    \centering
    \includegraphics[width=0.3\textwidth]{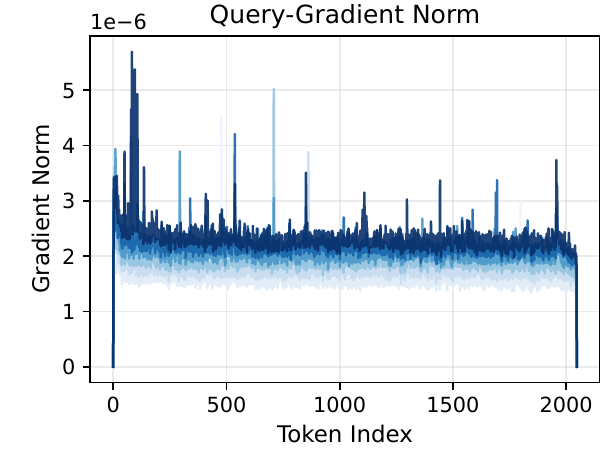}
    \includegraphics[width=0.3\textwidth]{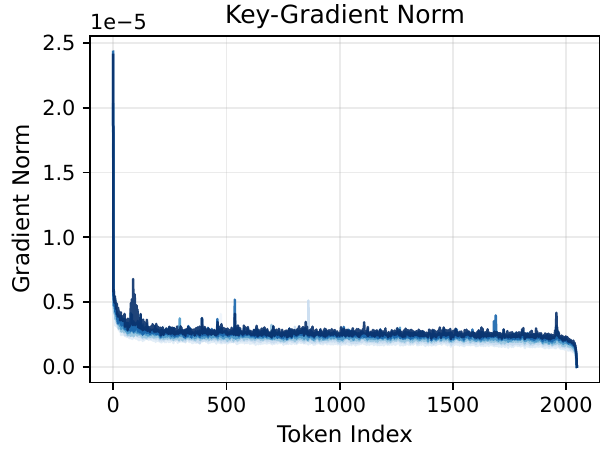}
    \includegraphics[width=0.3\textwidth]{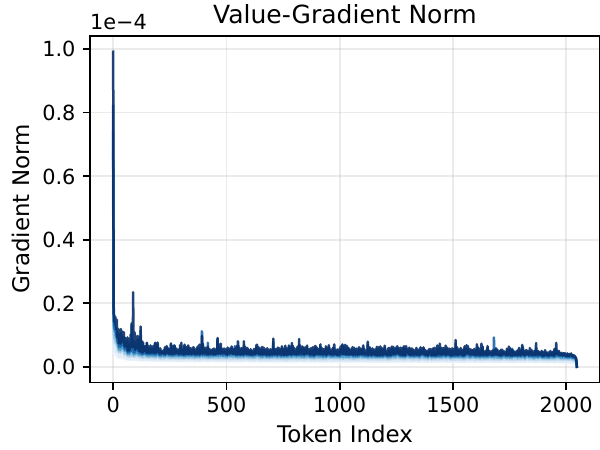}
    \includegraphics[width=0.06\textwidth]{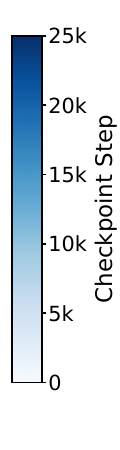}

    \caption{
    Token-wise gradient norms of query, key, and value across training checkpoints for the 1B model trained from scratch. Statistics are averaged over layers. Key and especially value gradients exhibit a pronounced spike at token 0, while query gradients remain comparatively flat.
    }
    \label{fig:qkv_grads_over_tokens}
\end{figure}

Figure~\ref{fig:qkv_grads_over_tokens} reveals a clear and highly structured asymmetry. On the key and value pathways, the gradient norm exhibits a pronounced spike at the first token across checkpoints, after which it drops rapidly and remains relatively flat over most of the sequence. Moreover, the value-path spike is consistently larger than the key-side one. By contrast, the query pathway looks qualitatively different: its token-wise gradient profile is much flatter, and the first-token query gradient is exactly zero due to causal masking. For each pathway, the statistics are averaged over layers.

These observations sharpen the empirical picture. Gradient sinks have a structured pathway profile, and the excess at the first token is carried mainly by the key and especially value gradients.

\subsection{Massive activations align with gradient reshaping}\label{sec:ma_grad_reshape}

We next ask how the pre-norm blocks transmit local gradients and how this reshaping relates to activation norms. Figure~\ref{fig:architecture_gradient}~(b) marks the gradient locations used in the following measurements. For a fixed layer and token, let $g_1 = \nabla_{h_t^{\ell+\nicefrac12}}\mathcal{L}$ be the gradient after the residual addition and $g_2 = \nabla_{\widetilde h_t^\ell}\mathcal{L}$ the gradient at the normalized attention input. Backpropagation through RMSNorm produces the branch contribution $g_3 = J_{\mathrm{RMSNorm}}(h_t^\ell)^\top g_2$. The skip path carries $g_1$ unchanged, so the total input gradient is $\nabla_{h_t^\ell}\mathcal{L} = g_1+g_3$.

Three stages of this backward flow are considered. \emph{Bloat} measures amplification through the attention branch from $g_1$ to $g_2$. \emph{Change} measures the net change in residual-stream gradient norm after the branch contribution is added. \emph{Compress} measures the norm gain through RMSNorm from $g_2$ to $g_3$. That is,
\begin{equation}\label{eq:ratios_attn}
\mathrm{Bloat}_t^{\mathrm{attn},\ell}
:= \frac{\|g_2\|_2}{\|g_1\|_2}, \quad
\mathrm{Change}_t^{\mathrm{attn},\ell}
:= \frac{\|g_1+g_3\|_2}{\|g_1\|_2}, \quad
\mathrm{Compress}_t^{\mathrm{attn},\ell}
:= \frac{\|g_3\|_2}{\|g_2\|_2}.
\end{equation}
The analogous definitions and measurements for the MLP branch appear in Appendix~\ref{app:mlp_reshaping}.

\begin{figure}[tb]
    \centering
    \includegraphics[width=0.3\linewidth]{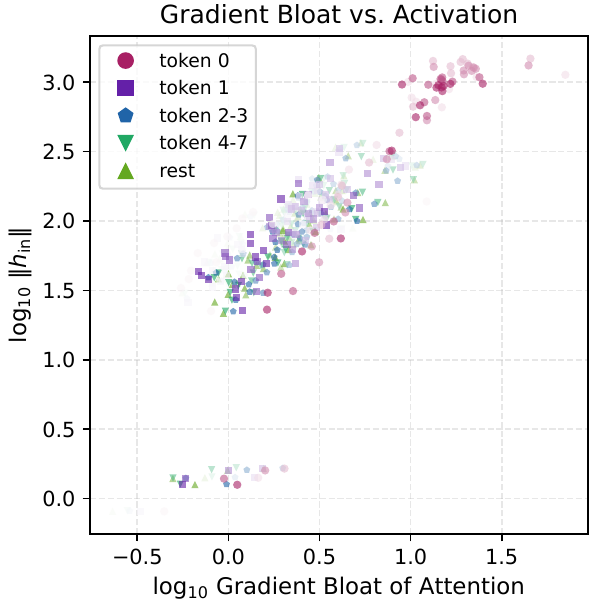}
    \includegraphics[width=0.3\linewidth]{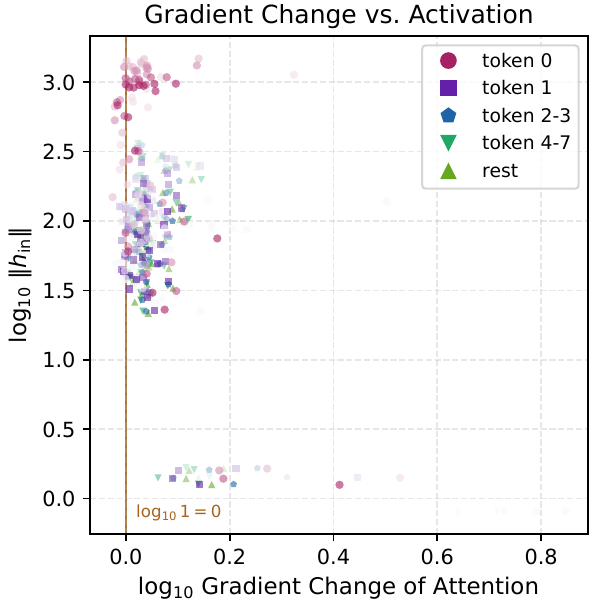}
    \includegraphics[width=0.3\linewidth]{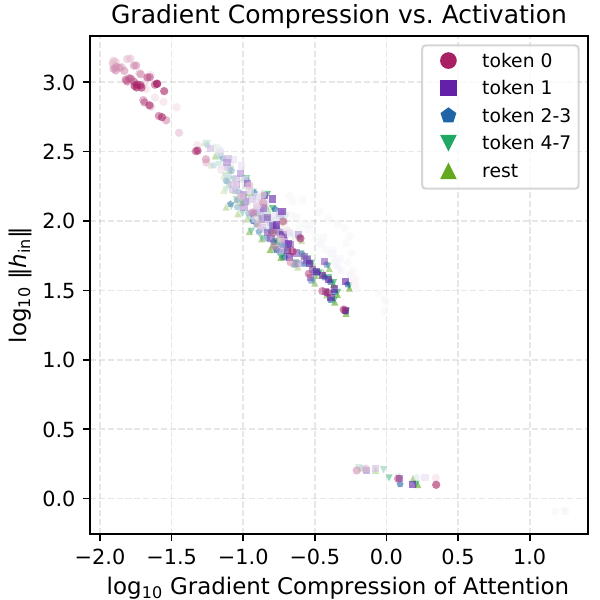}
    \includegraphics[width=0.06\linewidth]{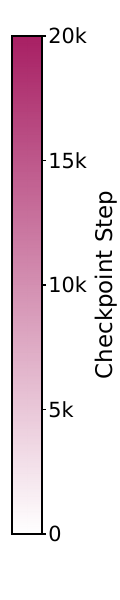}\\
    \includegraphics[width=0.3\linewidth]{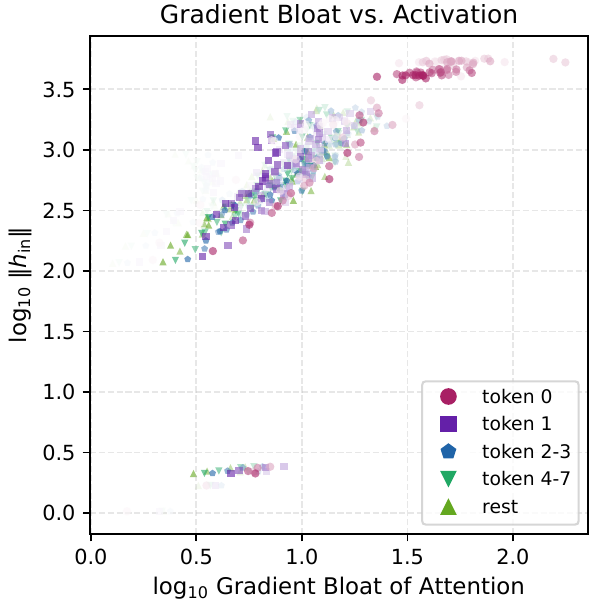}
    \includegraphics[width=0.3\linewidth]{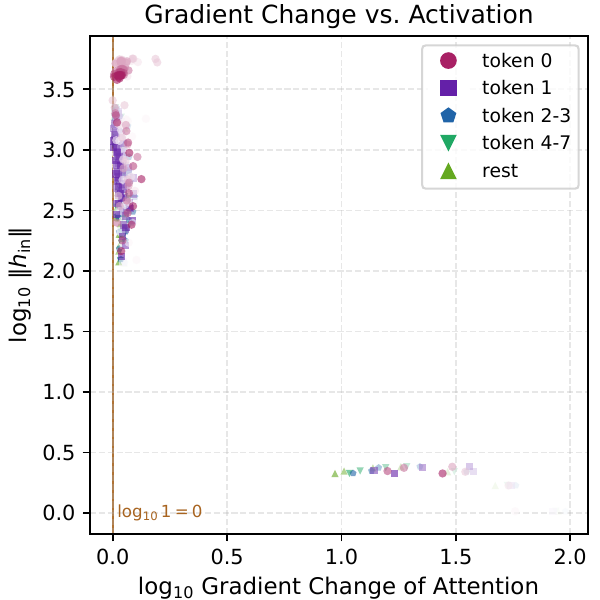}
    \includegraphics[width=0.3\linewidth]{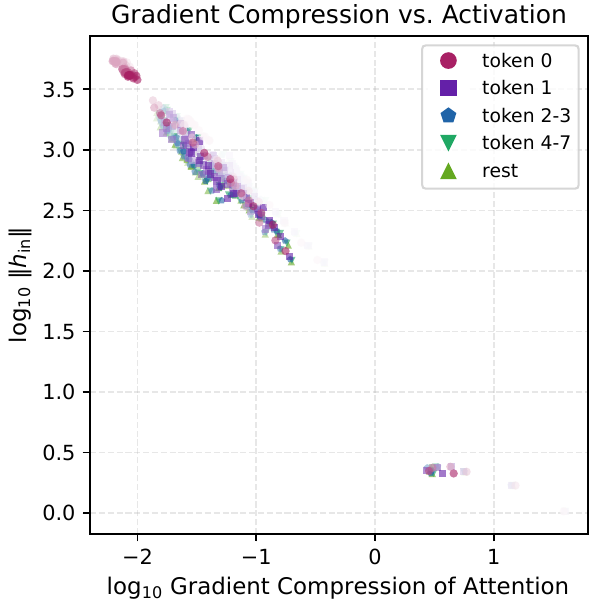}
    \includegraphics[width=0.06\linewidth]{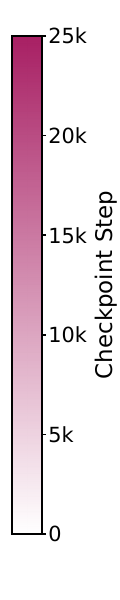}

    \caption{
    Scatter plots relating local gradient reshaping to input activation norms of the attention block. Top row: 0.1B model; bottom row: 1B model. From left to right, the panels show $\mathrm{Bloat}^{\mathrm{attn}}$, $\mathrm{Change}^{\mathrm{attn}}$, and $\mathrm{Compress}^{\mathrm{attn}}$ against $\|h_{\mathrm{in}}\|$ on a $\log_{10}$ scale. Each point is colored by token group (token 0, token 1, tokens 2--3, tokens 4--7, and the rest early tokens 8--15). 
    Large-activation points from token 0 populate the high-bloat regime, yet their residual-stream gradient norms change only mildly for most layers. The $\mathrm{Compress}$ ratio indicates stronger attenuation at massive activation sites.
    }\label{fig:attn_scatter_main}
\end{figure}

Figure~\ref{fig:attn_scatter_main} compares each gradient-reshaping ratio with the corresponding input activation norm averaged over evaluation batches.
The left panels show that large-activation points, almost all from token 0, are precisely where branch-local amplification becomes extreme. Such large $\mathrm{Bloat}^{\mathrm{attn}}$ naturally raise the concern about distortions on residual-stream transport during backpropagation that could harm training stability~\cite{xiong2020layernorm,liu2020understanding}.
The middle panels show that this is largely not the case. The $\mathrm{Change}^{\mathrm{attn}}$ values remain tightly concentrated near $1$ and therefore near $\log_{10}1 = 0$ on the plotted scale, aside from a small number of points outside the massive-activation regime.
Thus, the residual-stream gradient norm usually does not change dramatically. The right panels suggest why this remains possible. Larger activation norms are consistently associated with smaller $\mathrm{Compress}^{\mathrm{attn}}$, revealing a strong negative relationship between activation scale and local gradient attenuation.
That is, the most amplified gradients are also the most strongly compressed before being passed back into the residual stream, avoiding uncontrolled gradient growth.

Taken together, these measurements give a backward-pass view of massive activations on the first token. They suggest that these large norms are not merely forward numerical outliers, but mark sites where gradient concentrations are locally reshaped.

\section{Theoretical Mechanism}\label{sec:theory}

This section explains the mechanism behind the empirical patterns above, making it clear how attention sinks, gradient sinks, and massive activations are connected in pre-norm Transformers. The key ingredients are attention routing and the scale sensitivity of RMSNorm. We present the main identities and bounds here, with the full theory and proofs deferred to Appendix~\ref{app:theory}.

\subsection{Attention weights route gradients backward}

\begin{wrapfigure}[9]{r}{0.3\linewidth}
	\centering
    \vspace{-20pt}
    \includegraphics[width=1\linewidth]{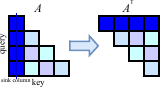}\caption{Attention weights as two-way routers.}
    \label{fig:as_motivation}
\end{wrapfigure}

We first formalize the basic intuition given in Figure~\ref{fig:as_motivation} that attention weights route not only values forward but also value-path gradients backward. Fix one layer and one attention head, and suppress their superscripts for simplicity. Let $A=(a_{tj})\in\mathbb{R}^{T\times T}$ be the causal attention matrix, with $a_{tj}=0$ for $j>t$. Stack the value states and head outputs row-wise as $V,Y\in\mathbb{R}^{T\times d_{\mathrm{head}}}$, so that $Y=AV$. Let $G:=\nabla_Y\mathcal{L}$ be the upstream gradient at $Y$, and write its $t$-th row as $g_t:=\nabla_{y_t}\mathcal{L}$. We now record the exact value-path identity induced by attention aggregation.

\begin{proposition}[Exact value-path aggregation]\label{prop:v_agg_main}
For one attention head, the value-path backward signal satisfies $\nabla_V \mathcal{L} = A^\top G$, or equivalently $\nabla_{v_s}\mathcal{L} = \sum\nolimits_{t=s}^{T-1} a_{ts} g_t$ for any token index $s$.
\end{proposition}

The identity shows the basic structural routing mechanism. A forward attention sink means that many later rows place substantial mass on one early column $s$ of $A$. On the value path, backpropagation uses the corresponding row $s$ of $A^\top$ and therefore aggregates the gradients of later tokens into the same early token. In other words, attention weights are two-way routers.

The full key and query identities are given in Appendix~\ref{app:theory:exact}, and they clarify the empirical asymmetry observed in Section~\ref{sec:as_to_gs}. Key gradients still depend on sink columns, while query gradients are row-local. In particular, under causal masking, $\nabla_{q_0}\mathcal{L}=0$ exactly.

\subsection{Sink statistics control localized gradient pressure}

We relate the value-path gradient to the attention-column statistics $M_s$ and $S_s$ defined in~\eqref{eq:sink_mass}.

\begin{theorem}[Value-path gradient control by sink statistics]\label{thm:v_second_moment_main}
For fixed model parameters, write $\mathbb E_A[\cdot]=\mathbb E[\cdot\mid A]$ given attention matrix $A$. Suppose that the upstream gradients on positions $t\ge s$ admit the decomposition $g_t=\mu+\varepsilon_t$ with $\mu=\mathbb{E}_A[g_t]\in\mathbb{R}^{d_{\mathrm{head}}}$ for all $t\ge s$. The noise terms satisfy $\mathbb{E}_A[\varepsilon_t]=0$, $\mathbb E_A\|\varepsilon_t\|_2^2\le\sigma^2$ for all $t$, and $|\mathbb E_A\langle\varepsilon_t,\varepsilon_{t'}\rangle|
\le\rho$ for all $t\ne t'$.
Then,
$M_s^2\|\mu\|_2^2
\le
\mathbb{E}_A\big[\|\nabla_{v_s}\mathcal{L}\|_2^2\big]
\le
M_s^2\|\mu\|_2^2 + \sigma^2 S_s + \rho(M_s^2 - S_s)$.
\end{theorem}

Theorem~\ref{thm:v_second_moment_main} shows that sink structure provides a natural and quantitatively explicit route to value-path gradient concentration. The term $M_s^2\|\mu\|_2^2$ is the coherent accumulation along the sink column, while $S_s$ and $\rho(M_s^2-S_s)$ bound the accumulated variance and cross-token covariance. For a nonzero coherent component, larger attention-column mass raises the lower bound on the conditional second moment of the value-path gradient.

The dominant value-path spike seen in Section~\ref{sec:as_to_gs} is thus not accidental. The value path is the pathway where sink statistics most transparently govern gradient aggregation. Appendix~\ref{app:theory:kq} further provides supporting results for the key and query pathways. The key bound has a similar dependence on $M_s$, whereas the query bound depends on within-row key dispersion.

\subsection{RMSNorm turns activation scale into gradient regulation}

It remains to connect this localized pressure to massive activations. In a pre-norm block, the relevant map is the RMSNorm applied before the attention or MLP branch.

\begin{theorem}[Activation-dependent compression under RMSNorm]\label{thm:rms_compress_main}
Denote $y=\mathrm{RMSNorm}(x) = \gamma \odot \frac{x}{\mathrm{rms}(x)}$ with $\mathrm{rms}(x)=\sqrt{\frac{1}{d}\|x\|_2^2+\epsilon_{\mathrm{rms}}}$, where $x\in \mathbb{R}^d$ is the input and $\gamma\in \mathbb{R}^d$ is the scale parameter. Let $g_y=\nabla_y\mathcal{L}$ be the upstream gradient at the RMSNorm output.
Then we have $\nabla_x\mathcal{L}=J_{\mathrm{rms}}(x)^\top g_y$, and the Jacobian satisfies $\|J_{\mathrm{rms}}(x)\|_{\mathrm{op}} \le \frac{\|\gamma\|_\infty}{\mathrm{rms}(x)}$.
\end{theorem}

Theorem~\ref{thm:rms_compress_main} formalizes that a larger token-wise activation norm directly lowers the worst-case gain from upstream gradients to earlier layers. Consequently, for any target gradient scale $\tau>0$, the condition $\mathrm{rms}(x)\ge \frac{\|\gamma\|_\infty}{\tau}\|g_y\|_2$ is sufficient to ensure $\|\nabla_x\mathcal{L}\|_2\le\tau$. This is the mathematical sense in which large activations can act as gradient regulators rather than merely as pathological outliers.

Combining this with Proposition~\ref{prop:v_agg_main} and Theorem~\ref{thm:v_second_moment_main}, attention sinks create localized value-path pressure, and RMSNorm makes activation scale an available local attenuation factor. Massive activations are therefore well positioned to emerge in response to sink-induced pressure, helping preserve stable residual-stream gradient transport by absorbing the pressure created by gradient sinks.

\section{Mechanistic Intervention and Validation}
\label{sec:vscale}

\begin{wrapfigure}[22]{r}{0.5\linewidth}
	\centering
    \vspace{-10px}
	\includegraphics[width=1\linewidth]{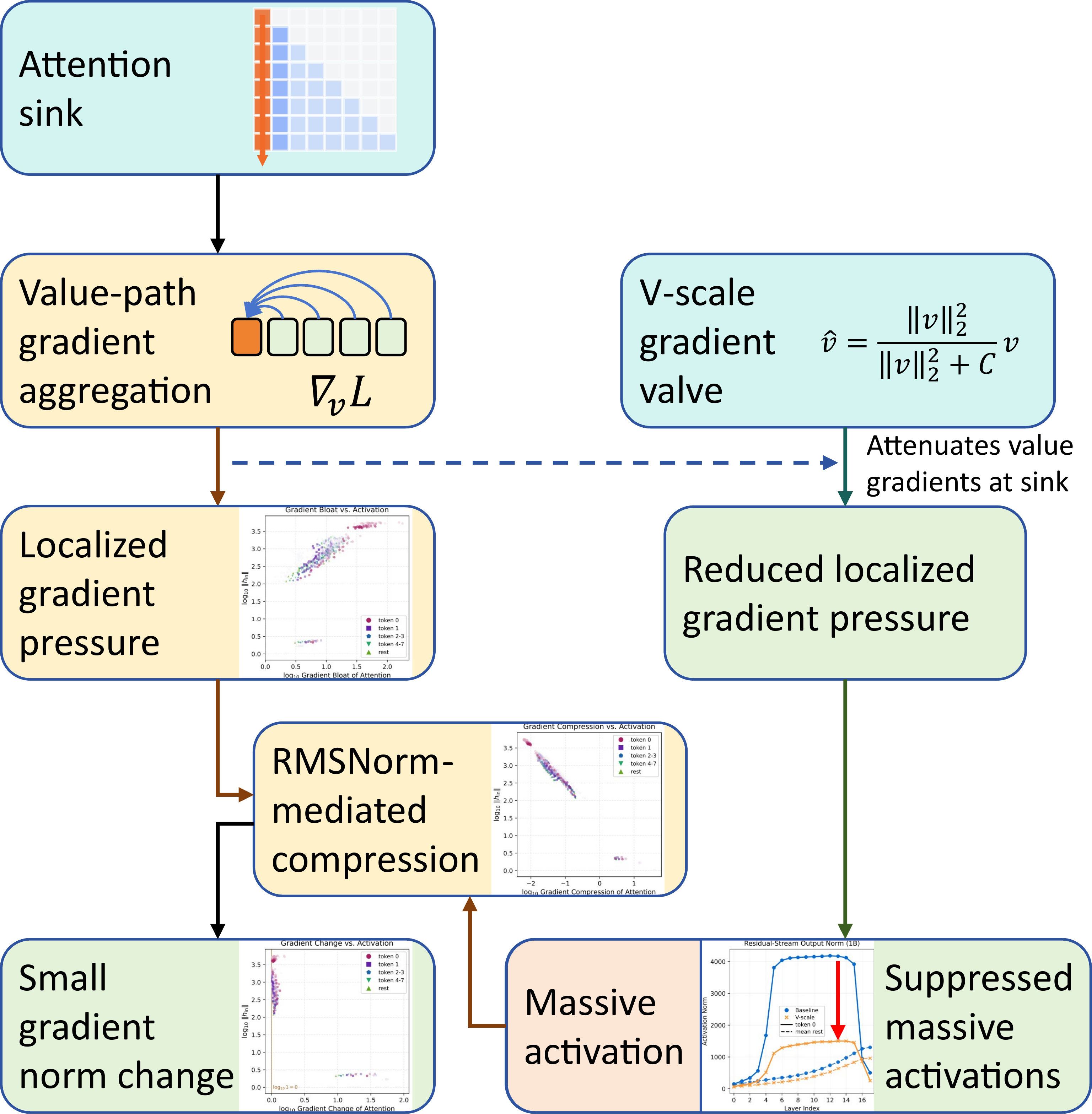}
    \vspace{-10px}
	\caption{
		Mechanism and intervention.
		MA attenuates sink-induced gradient pressure through RMSNorm. V-scale provides an additional value-path gradient valve, reducing reliance on MA.
	}\label{fig:mechanism_overview}
\end{wrapfigure}

The proposed mechanism predicts that relieving localized gradient pressure should weaken MA while retaining AS. We test this prediction with V-scale and a backward-only control in matched pretraining runs, then evaluate language modeling and long-context retrieval. Figure~\ref{fig:mechanism_overview} summarizes the mechanism and intervention rationale.

The value path is the natural target channel to regulate gradients for three reasons.
First, modifying only value states does not affect the attention weights that depend only on $Q$ and $K$.
Second, our observations in Section~\ref{sec:as_to_gs} confirm that value-path gradient norms of the sink token are several times larger than key-path ones. 
Third, prior studies have repeatedly reported that sink tokens tend to have unusually small value norms~\cite{guo2025activedormant,su2025kvsink,bu2025valuestategated}.
We observe the same pattern in our models, motivating a radial transform that uses $\|v\|_2$ to target sink tokens.

\subsection{V-scale: a value-path gradient valve}

We now introduce the intervention used to test the mechanism.
For each attention head, after the standard value projection and before value aggregation, we replace $y_t = \sum\nolimits_{j\le t} a_{tj} v_j$ by
\[
\hat v_j = \phi(\|v_j\|_2^2)\, v_j,
\quad
\phi(r)=\frac{r}{r+C}, \quad
\hat y_t = \sum\nolimits_{j\le t} a_{tj} \hat v_j,
\]
where $C>0$ is a scale parameter. We call this modification \emph{V-scale}, shown as value path 2 in Figure~\ref{fig:architecture_gradient}~(c).
We parameterize $C_{\ell,h}=d_{\mathrm{model}}\cdot d_{\mathrm{head}}\cdot\sigma^2\exp(\theta_{\ell,h})$, where $\sigma=0.02$ is the value-projection initialization standard deviation and $\theta_{\ell,h}$ starts at $0$. This adds one learned parameter per layer and head and matches the typical squared value norm at initialization (see Appendix~\ref{app:theory:vscale}). The backward behavior of V-scale is explicit:

\begin{proposition}[Jacobian spectrum of V-scale]\label{prop:vscale_valve_main}
Let $r=\|v\|_2^2$ and $\hat v=\phi(r)v$ with $\phi(r)=\frac{r}{r+C}$.
Then the Jacobian of the V-scale map $v\mapsto \hat v$ is $J_{\phi}(v) = \phi(r) I + \frac{2C}{(r+C)^2} vv^\top$ with eigenvalue $\lambda_\perp(r)=\frac{r}{r+C}$ on the $(d_{\mathrm{head}}-1)$-dimensional subspace orthogonal to $v$ and $\lambda_\parallel(r) = \frac{r^2+3Cr}{(r+C)^2}$ along the radial direction $v$.
\end{proposition}

Full derivations are given in Appendix~\ref{app:theory:vscale}.
Proposition~\ref{prop:vscale_valve_main} makes the backward effect of V-scale explicit. If $\|v\|_2^2\ll C$ for the sink token, the value-path backward signal is strongly attenuated. Conversely, both eigenvalues are close to $1$ when $\|v\|_2^2\gg C$, so V-scale nearly preserves gradient transmission on large-norm value states.
V-scale also contracts small-norm value states in the forward pass. We separate this effect from gradient modulation with the backward-only control below.

Several recent architectural variants mitigate AS or MA by changing the softmax function, inserting gates, or modifying normalization structure~\cite{zuhri2025softpick,qiu2025gated,bu2025valuestategated,sun2026anatomy}.
Unlike these broader architectural changes, V-scale is intended as a mechanism test that adds a transparent value-path gradient valve.

\paragraph{Backward-only control}
Full V-scale changes both the forward values and their backward Jacobian. To isolate the latter, we also test a backward-only control of V-scale
\begin{equation}\label{eq:vscale_backward_only}
    \hat v_{\mathrm{BO}}
    = \operatorname{sg}(v)
    + \phi(\|v\|_2^2)v
    - \operatorname{sg}\!\left(\phi(\|v\|_2^2)v\right),
\end{equation}
where $\operatorname{sg}$ stops gradients and is implemented with \texttt{detach}. The scaled terms cancel in the forward pass, giving $\hat v_{\mathrm{BO}}=v$. During backpropagation, autodiff applies $J_\phi(v)$ from Proposition~\ref{prop:vscale_valve_main}. We freeze $\theta_{\ell,h}=0$ throughout this run to avoid updating a parameter that has no effect on the forward value.

\subsection{Experiments}

We train V-scale models using the same data, optimizer, and model configurations as the corresponding baselines. We first verify the mechanism using token-wise gradient and activation measurements, and then evaluate models on downstream tasks.

\begin{figure}[tb]
    \centering
    \includegraphics[width=0.48\textwidth]{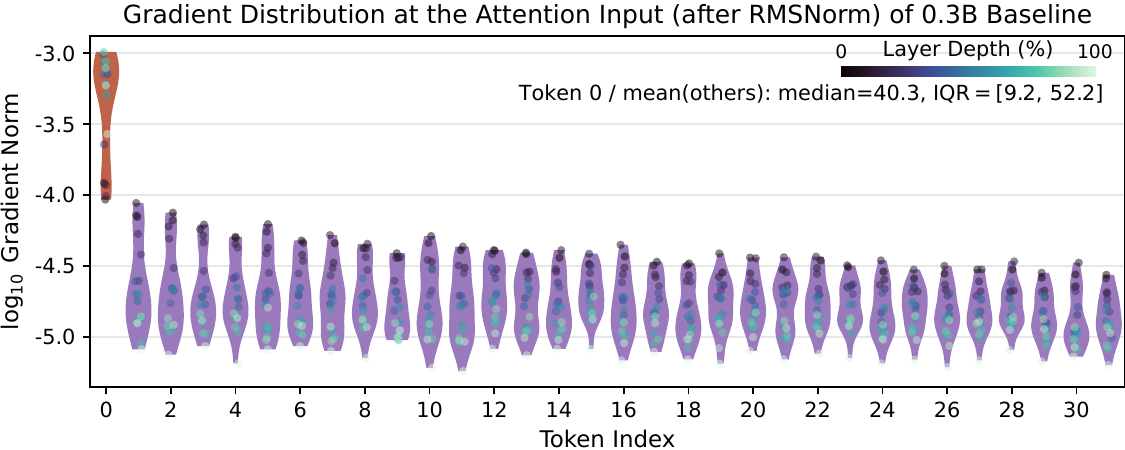}
    \includegraphics[width=0.48\textwidth]{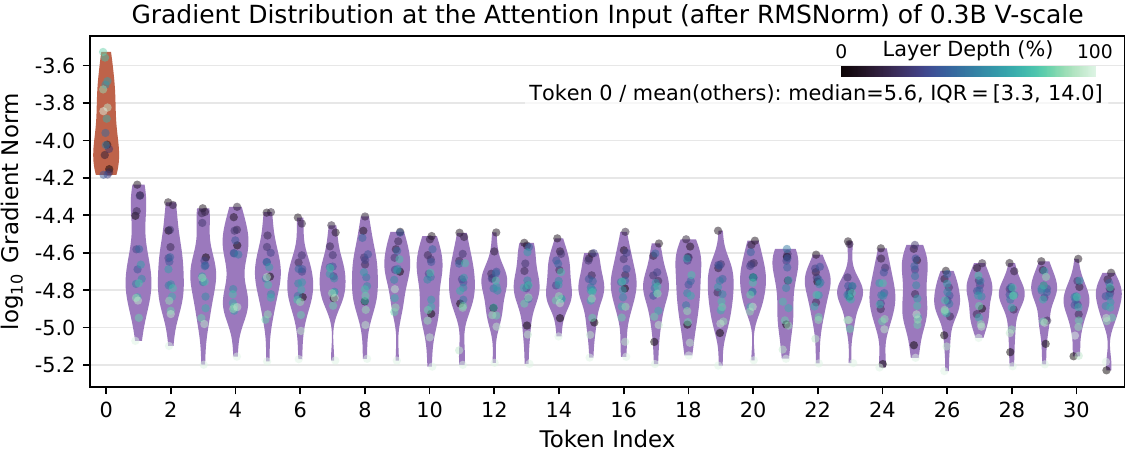}

    \caption{
    Gradient sinks in the 0.3B baseline and V-scale models. Each panel plots the token-wise distribution of $\log_{10}\|\nabla_{\widetilde h_t^{\ell}}\mathcal L\|_2$ at the post-RMSNorm attention input. Token 0 is highlighted. V-scale still leaves a visible gradient sink but reduces the concentration relative to the baseline.
    }
    \label{fig:vscale-gs}
\end{figure}

\begin{figure}[tb]
    \centering
    \includegraphics[width=0.24\linewidth]{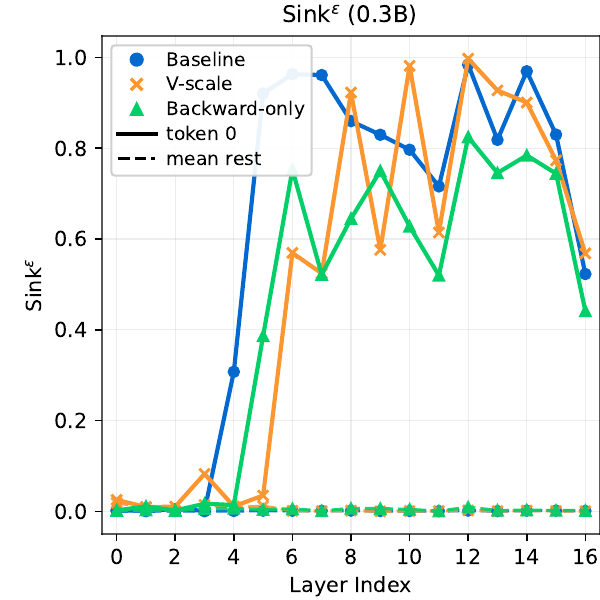}
    \includegraphics[width=0.24\linewidth]{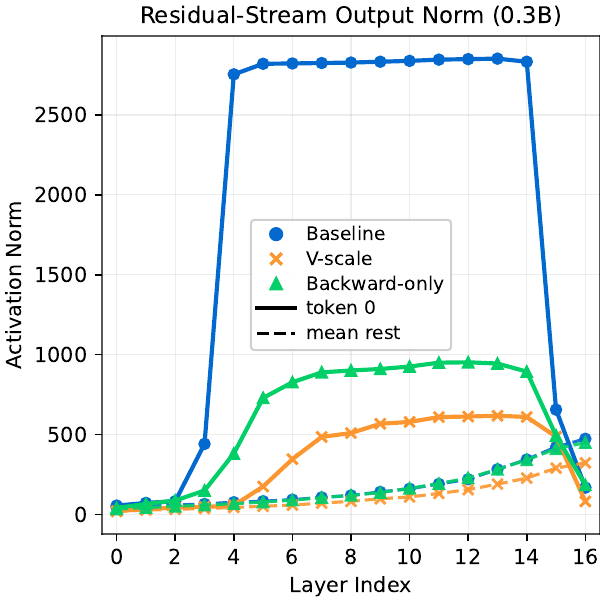}
    \includegraphics[width=0.24\linewidth]{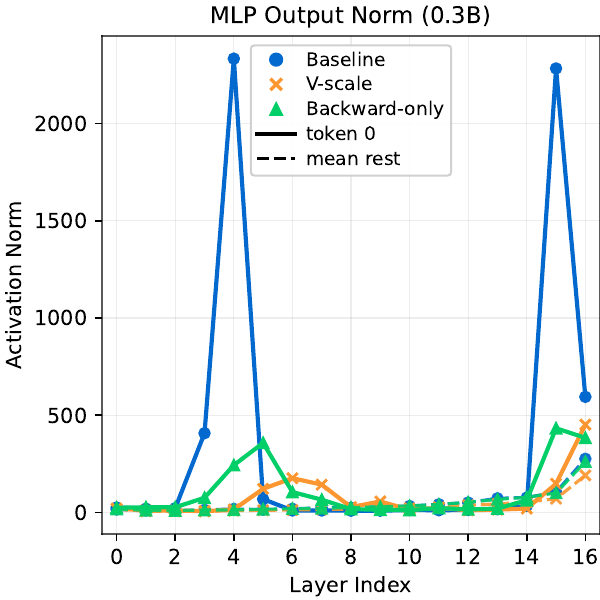}
    \includegraphics[width=0.24\linewidth]{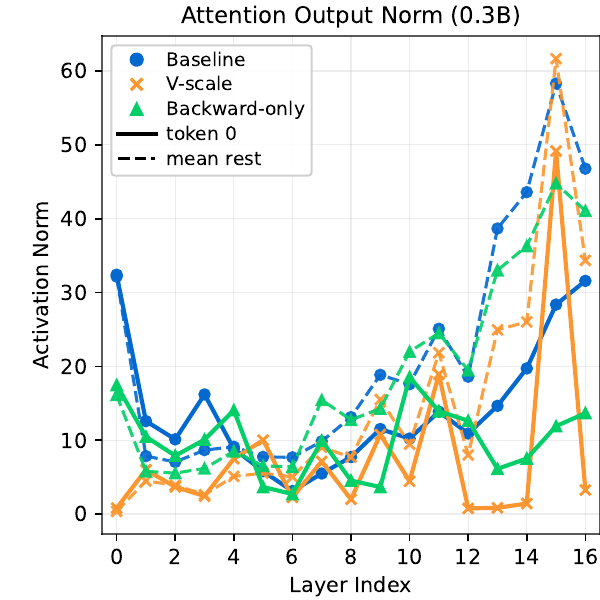}
    \\
    \includegraphics[width=0.24\linewidth]{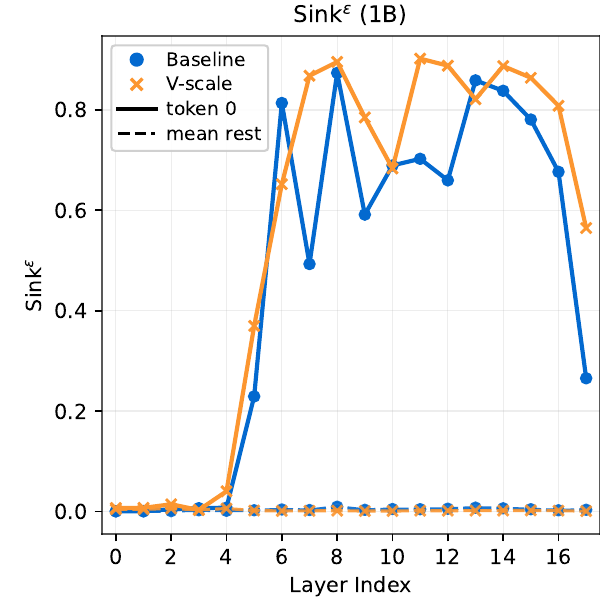}
    \includegraphics[width=0.24\linewidth]{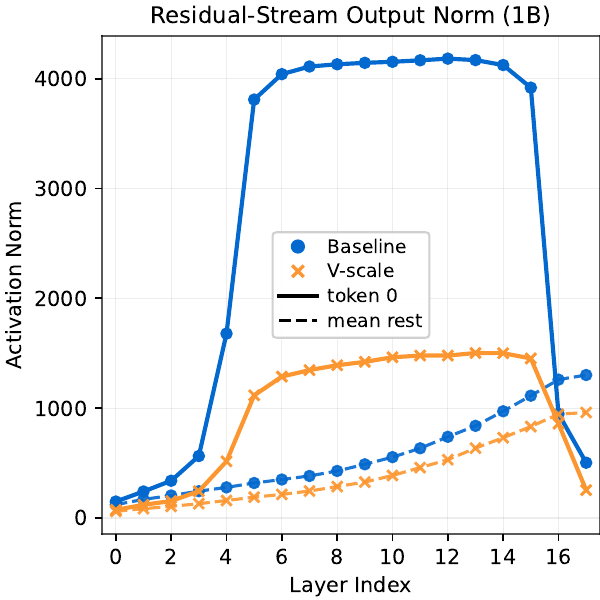}
    \includegraphics[width=0.24\linewidth]{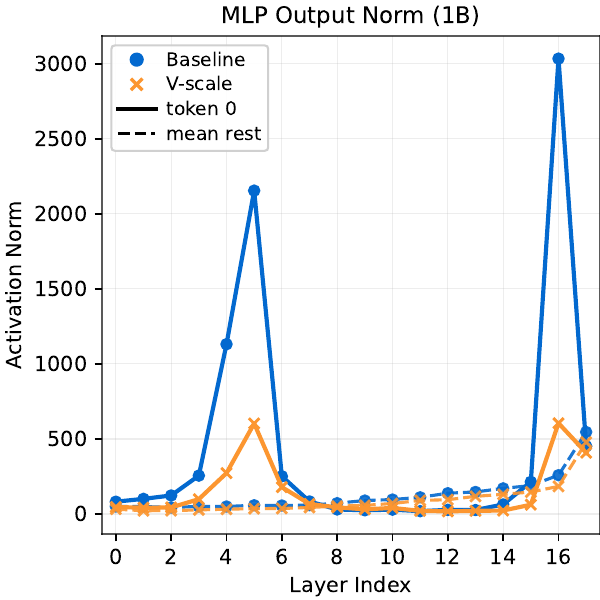}
    \includegraphics[width=0.24\linewidth]{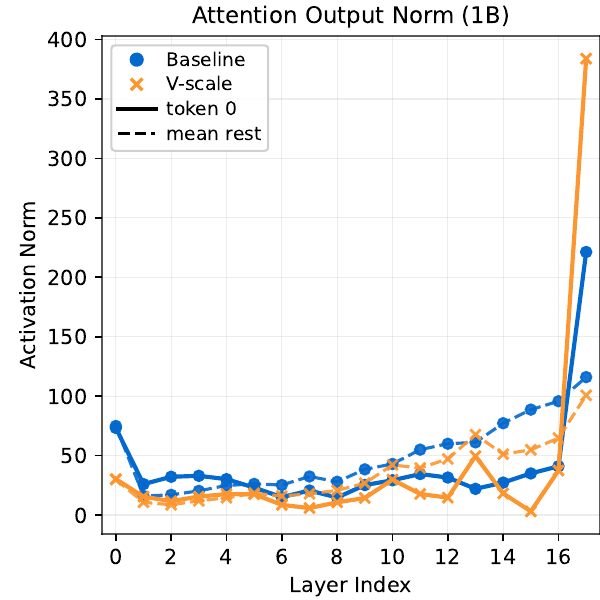}
    
    \caption{
    Forward phenomena in baseline, full V-scale, and backward-only models.
    Top row: 0.3B models; bottom row: the 1B baseline and full V-scale models.
    From left to right: thresholded sink rate, residual-stream output norm, MLP output norm, and attention output norm.
    We compare the first token with the mean over the rest early tokens (positions 1--15).
    Across both scales, V-scale largely preserves attention sinks while reducing the activation norms of token 0, with the larger reduction appearing in MLP outputs. The backward-only control shows the same qualitative pattern.
    }
    \label{fig:vscale_as_ma_compare}
\end{figure}

\paragraph{Mechanism verification}
We re-measure gradient sinks as defined in Definition~\ref{def:gradient_sink}. Figure~\ref{fig:vscale-gs} shows that in the 0.3B V-scale model, the gradient sink remains present but is substantially weaker, with the median $R_{GS}^{\ell}(0)$ defined in~\eqref{eq:gs_def} dropping from roughly $40.3$ in the baseline to $5.6$. This is the expected outcome of a partial intervention on the value-path gradient.

Moreover, Figure~\ref{fig:vscale_as_ma_compare} displays the layer-wise thresholded Sink Rate $\mathrm{Sink}_s^{\epsilon,\ell}$ in~\eqref{eq:sink_rate} and output norms at several sites. The top row compares the 0.3B baseline, full V-scale, and backward-only models, while the bottom row compares the 1B baseline and full V-scale models. The same AS and MA protocol is used for all models, averaging over a batch of 512 inputs with 2048 tokens.

Across both model scales, V-scale preserves strong AS behavior and in several layers even slightly strengthens it. In contrast, on the MA metrics, the residual-stream and MLP output norms of token 0 are clearly reduced relative to the baselines. Attention output norms also decrease, consistent with the direct contraction of value states, but the much stronger effect appears in MLP outputs where V-scale has no direct forward scaling.

The backward-only control produces a similar reduction in MA while retaining substantial AS, as shown in the top row of Figure~\ref{fig:vscale_as_ma_compare}. These results show that intervening the value-path backward rule alone is sufficient to suppress MA during training.

Overall, these observations match the qualitative prediction of our mechanism. By providing an additional gradient valve on the dominant value path, V-scale reduces the pressure for the model to realize MA through large MLP outputs.

\paragraph{Downstream performance}
Beyond mechanism verification, we next examine how V-scale behaves on downstream evaluations. We evaluate the 1B baseline and full V-scale models on standard language-modeling and reasoning tasks using the LM Evaluation Harness~\cite{eval-harness}. This serves as a basic capability check, since suppressing massive activations should not come at the cost of ordinary downstream behavior. Table~\ref{tab:main_standard_1b} shows comparable overall performance under \texttt{bfloat16}. V-scale substantially improves ARC-C and BoolQ, while some other tasks show small decreases.

\begin{table}[t]
    \centering
    \caption{
    Representative standard downstream evaluation of 1B baseline and V-scale models under \texttt{bfloat16}.
    Lower perplexity (ppl) is better, while higher accuracy (acc) is better.
    Bold marks the better result within each baseline and V-scale pair under the same evaluation setting.
    }\label{tab:main_standard_1b}
	\setlength\tabcolsep{3pt}
    \begin{tabular}{lcccccccccc}
        \toprule
        \multirow{2}{*}{Model} & Wiki & \multicolumn{2}{c}{Lambada} & ARC-C & BoolQ & COPA & HellaSwag & OBQA & PIQA & WinoGrande \\
        & ppl$\downarrow$ & ppl$\downarrow$ & acc$\uparrow$ & acc$\uparrow$ & acc$\uparrow$ & acc$\uparrow$ & acc$\uparrow$ & acc$\uparrow$ & acc$\uparrow$ & acc$\uparrow$ \\
        \midrule
        Baseline & 22.84 & \textbf{15.27} & 43.28 & 24.91 & 51.38 & 69.00 & \textbf{51.59} & 31.40 & \textbf{72.52} & \textbf{54.22} \\
        V-scale  & \textbf{22.83} & 15.91 & \textbf{43.37} & \textbf{27.05} & \textbf{56.27} & \textbf{70.00} & 51.29 & \textbf{32.80} & 71.49 & 53.28 \\
        \bottomrule
    \end{tabular}
\end{table}

As a further behavioral diagnostic, we evaluate Needle-in-a-Haystack (NIAH) retrieval under \texttt{bfloat16} and four post-training quantization methods: SmoothQuant (W8A8)~\cite{xiao2023smoothquant}, BitsAndBytes (BNB, W4A16)~\cite{dettmers2023qlora}, GPTQ (W4A16)~\cite{frantar2023optq}, and AWQ (W4A16)~\cite{lin2024awq}, where W$x$A$y$ denotes $x$-bit weights and $y$-bit activations. Table~\ref{tab:main_niah_multikey_1b} reports mean accuracy of multi-key NIAH over $5$ evaluation seeds. V-scale outperforms the baseline in every tested combination of precision and context length, both within the native context and after YaRN extension~\cite{peng2024yarn}. At 8192 tokens, the advantages range from $8.68$ to $15.04$ percentage points. These consistent gains demonstrate improved retrieval when multiple key-value associations compete within a long context.

There are mixed results on the harder single-needle tasks shown in Appendix~\ref{app:additional_downstream}. The performance drops in some quantization settings after YaRN extension, but the causes of these interactions remain unresolved. The gains on multi-key NIAH, BoolQ, and ARC-C may reflect their shared demand for selecting relevant information under competing cues. Learned V-scale parameters vary across layers and heads (see Figure~\ref{fig:vscale_theta_heatmap} in Appendix~\ref{app:support_observation}), but connecting this variation to performance gains requires further analysis of the affected head functions.

\paragraph{Sensitivity to activation quantization}
To examine sensitivity to activation range more directly, we compare symmetric W8A8 quantization with dynamic per-token and calibrated static per-tensor activation scales. The detailed results are deferred to Appendix~\ref{app:activation_quantization}. Dynamic scaling keeps both models close to their \texttt{bfloat16} performance. However, under the more stressful static scaling, baseline retrieval falls to near-zero accuracy across all NIAH variants at 2048 tokens, whereas V-scale retains substantial retrieval capability. Quantizing only the MLP \texttt{down\_proj} modules nearly reproduces this contrast. These results show that V-scale is less sensitive to a coarse, fixed activation range along the MLP pathway associated with its strongest MA reduction.

\begin{table}[t]
    \centering
    \setlength\tabcolsep{2.5pt}
    \caption{
    Main multi-key Needle-in-a-Haystack result for 1B models.
    Entries are mean retrieval accuracy percentages over $5$ random seeds; higher is better.
    We directly compare the baseline (B) and V-scale (V) models under native-context evaluation and YaRN extension with a factor of $4$.
    Bold marks the better result within each baseline and V-scale pair under the same evaluation setting.
    }\label{tab:main_niah_multikey_1b}
    \begin{tabular}{l
        @{\hspace{6pt}}cc
        @{\hspace{6pt}}cc
        @{\hspace{6pt}}cc
        @{\hspace{6pt}}cc
        @{\hspace{6pt}}cc
        @{\hspace{6pt}}cc}
        \toprule
        & \multicolumn{4}{c}{Native}
        & \multicolumn{8}{c}{YaRN-extension} \\
        \cmidrule(lr){2-5}\cmidrule(lr){6-13}
        \multirow{2}{*}{Quantization} & \multicolumn{2}{c}{1024}
        & \multicolumn{2}{c}{2048}
        & \multicolumn{2}{c}{1024}
        & \multicolumn{2}{c}{2048}
        & \multicolumn{2}{c}{4096}
        & \multicolumn{2}{c}{8192} \\
        \cmidrule(lr){2-3}\cmidrule(lr){4-5}\cmidrule(lr){6-7}\cmidrule(lr){8-9}\cmidrule(lr){10-11}\cmidrule(lr){12-13}
        & B & V & B & V & B & V & B & V & B & V & B & V \\
        \midrule
        \texttt{bfloat16} & 67.96 & \textbf{73.88} & 65.04 & \textbf{71.64} & 57.76 & \textbf{59.28} & 53.88 & \textbf{59.72} & 48.00 & \textbf{56.16} & 48.32 & \textbf{62.00} \\
        SmoothQuant & 64.44 & \textbf{71.64} & 61.64 & \textbf{69.60} & 52.44 & \textbf{57.92} & 49.64 & \textbf{57.92} & 45.76 & \textbf{52.88} & 48.96 & \textbf{57.64} \\
        BNB & 66.16 & \textbf{74.04} & 61.60 & \textbf{70.56} & 54.76 & \textbf{58.00} & 46.28 & \textbf{58.40} & 42.72 & \textbf{55.88} & 47.56 & \textbf{62.60} \\
        GPTQ & 66.20 & \textbf{67.88} & 63.00 & \textbf{65.56} & 56.28 & \textbf{56.68} & 50.24 & \textbf{57.08} & 45.20 & \textbf{51.16} & 46.72 & \textbf{58.96} \\
        AWQ & 63.96 & \textbf{66.84} & 61.68 & \textbf{67.76} & 54.04 & \textbf{59.24} & 49.32 & \textbf{59.00} & 42.52 & \textbf{49.64} & 47.04 & \textbf{58.48} \\
        \bottomrule
    \end{tabular}
\end{table}

\section{Conclusion}
\label{sec:conclusion}

In this work, we revisit the relationship between attention sinks and massive activations from the perspective of backpropagation empirically and theoretically.
Under causal masking, AS can induce concentrated gradient pressure on the sink token, forming gradient sinks, and MA emerges as an adaptive regulator of this localized pressure in pre-norm Transformers.
To test this hypothesis, we introduce V-scale, a value-path intervention that modulates backpropagated gradients, and a backward-only control that removes its direct forward-value change.
Across model sizes, we find that V-scale preserves AS but substantially suppresses MA, illustrating that the two phenomena can be separated once gradient sinks are controlled.
The control retains substantial AS while suppressing MA, showing that backward modulation alone is sufficient for the predicted effect.
On downstream tasks, V-scale remains comparable to the baseline on standard evaluations and improves multi-key retrieval across all tested context lengths and precision settings.

Our findings motivate studying how sink-induced gradient regulation interacts with numerical precision during training~\cite{qiu2026flash}. A broader direction is to examine whether interventions on this mechanism preserve in-context learning~\cite{wu2025icl} and data-efficient agentic learning~\cite{wang2026deal}. These connections call for direct evaluations of few-shot learning and agent experience reuse.

\paragraph{Limitations}
This work has some limitations. First, our study is restricted to Llama-like dense language models. It remains unclear whether and to what extent the same mechanism transfers to other settings, such as mixture-of-experts architectures or multimodal Transformers. Second, our analysis does not exclude other factors, such as data distribution, optimizer, or numerical precision, from shaping AS and MA during training. These limitations are left for future work.

\section*{Acknowledgements}

Q. Yao is sponsored by Beijing Nova Program.
Z. Lin was supported by the NSF China (No. 62276004), the Beijing Natural Science Foundation (No. L257007), and the Beijing Major Science and Technology Project (No. Z251100008425006).

\clearpage
\bibliographystyle{abbrvnat}
\bibliography{neurips_2026}

@InProceedings{vaswani2017attention,
    author    = {Vaswani, Ashish and Shazeer, Noam and Parmar, Niki and Uszkoreit, Jakob and Jones, Llion and Gomez, Aidan N. and Kaiser, Lukasz and Polosukhin, Illia},
    booktitle = {NeurIPS},
    title     = {Attention is All you Need},
    year      = {2017},
}

@Article{touvron2023llama,
    author  = {Touvron, Hugo and Martin, Louis and Stone, Kevin and Albert, Peter and Almahairi, Amjad and Babaei, Yasmine and Bashlykov, Nikolay and Batra, Soumya and Bhargava, Prajjwal and Bhosale, Shruti and others},
    title   = {Llama 2: Open Foundation and Fine-Tuned Chat Models},
    year    = {2023},
    journal = {arXiv:2307.09288},
}

@Article{su2021rope,
    author  = {Su, Jianlin and Ahmed, Murtadha and Lu, Yu and Pan, Shengfeng and Bo, Wen and Liu, Yunfeng},
    title   = {RoFormer: Enhanced Transformer with Rotary Position Embedding},
    year    = {2024},
    journal = {Neurocomputing},
}

@inproceedings{zhang2019rmsnorm,
    title     = {Root Mean Square Layer Normalization},
    author    = {Biao Zhang and Rico Sennrich},
    year      = {2019},
    booktitle = {NeurIPS},
}

@inproceedings{xiong2020layernorm,
    title     = {On Layer Normalization in the Transformer Architecture},
    author    = {Ruibin Xiong and Yunchang Yang and Di He and Kai Zheng and Shuxin Zheng and Chen Xing and Huishuai Zhang and Yanyan Lan and Liwei Wang and Tie{-}Yan Liu},
    year      = {2020},
    booktitle = {ICML},
}

@inproceedings{xiao2024streamingllm,
    title     = {Efficient Streaming Language Models with Attention Sinks},
    author    = {Guangxuan Xiao and Yuandong Tian and Beidi Chen and Song Han and Mike Lewis},
    year      = {2024},
    booktitle = {ICLR},
}

@InProceedings{bondarenko2023quantizable,
    author    = {Bondarenko, Yelysei and Nagel, Markus and Blankevoort, Tijmen},
    title     = {Quantizable Transformers: Removing Outliers by Helping Attention Heads Do Nothing},
    booktitle = {NeurIPS},
    year      = {2023},
}

@InProceedings{sun2024massive,
    author    = {Sun, Mingjie and Chen, Xinlei and Kolter, J. Zico and Liu, Zhuang},
    title     = {Massive Activations in Large Language Models},
    booktitle = {COLM},
    year      = {2024},
}

@InProceedings{guo2025activedormant,
    author    = {Guo, Tianyu and Pai, Druv and Bai, Yu and Jiao, Jiantao and Jordan, Michael I. and Mei, Song},
    title     = {Active-Dormant Attention Heads: Mechanistically Demystifying Extreme-Token Phenomena in LLMs},
    booktitle = {CPAL},
    year      = {2025},
}

@InProceedings{zuhri2025softpick,
    title = {Softpick: No Attention Sink, No Massive Activations with Rectified Softmax},
    author = {Zuhri, Zayd Muhammad Kawakibi and Fuadi, Erland Hilman and Aji, Alham Fikri},
    booktitle = {Findings of ACL},
    year = {2026},
}

@InProceedings{qiu2025gated,
    author    = {Qiu, Zihan and Wang, Zekun and Zheng, Bo and Huang, Zeyu and Wen, Kaiyue and Yang, Songlin and Men, Rui and Yu, Le and Huang, Fei and Huang, Suozhi and others},
    title     = {Gated Attention for Large Language Models: Non-linearity, Sparsity, and Attention-Sink-Free},
    booktitle = {NeurIPS},
    year      = {2025},
}

@Article{qiu2026unified,
    author  = {Qiu, Zihan and Huang, Zeyu and Wen, Kaiyue and Jin, Peng and Zheng, Bo and Zhou, Yuxin and Huang, Haofeng and Wang, Zekun and Li, Xiao and Zhang, Huaqing and others},
    title   = {A Unified View of Attention and Residual Sinks: Outlier-Driven Rescaling is Essential for Transformer Training},
    journal = {arXiv:2601.22966},
    year    = {2026},
}

@InProceedings{sun2026anatomy,
    title={Anatomy of Massive Activations and Attention Sinks},
    author={Shangwen Sun and Alfredo Canziani and Yann LeCun and Jiachen Zhu},
    booktitle={ICML},
    year={2026},
}

@Article{mandt2017sgd,
    author  = {Mandt, Stephan and Hoffman, Matthew D. and Blei, David M.},
    title   = {Stochastic Gradient Descent as Approximate Bayesian Inference},
    journal = {JMLR},
    year    = {2017},
}

@Article{mccandlish2018noise,
    author  = {McCandlish, Sam and Kaplan, Jared and Amodei, Dario and {OpenAI Dota Team}},
    title   = {An Empirical Model of Large-Batch Training},
    journal = {arXiv:1812.06162},
    year    = {2018},
}

@Article{shazeer2020gluvariants,
    author  = {Shazeer, Noam},
    title   = {GLU Variants Improve Transformer},
    journal = {arXiv:2002.05202},
    year    = {2020},
}

@InProceedings{queipo-de-llano2026attention,
    author    = {Queipo-de-Llano, Enrique and Arroyo, Alvaro and Barbero, Federico and Dong, Xiaowen and Bronstein, Michael M. and LeCun, Yann and Shwartz-Ziv, Ravid},
    title     = {Attention Sinks and Compression Valleys in LLMs are Two Sides of the Same Coin},
    booktitle = {ICLR},
    year      = {2026},
}

@InProceedings{su2025kvsink,
    author    = {Su, Zunhai and Yuan, Kehong},
    title     = {KVSink: Understanding and Enhancing the Preservation of Attention Sinks in KV Cache Quantization for LLMs},
    booktitle = {COLM},
    year      = {2025},
}

@InProceedings{gu2025when,
    author    = {Gu, Xiangming and Pang, Tianyu and Du, Chao and Liu, Qian and Zhang, Fengzhuo and Du, Cunxiao and Wang, Ye and Lin, Min},
    title     = {When Attention Sink Emerges in Language Models: An Empirical View},
    booktitle = {ICLR},
    year      = {2025},
}

@InProceedings{barbero2025why,
    author    = {Barbero, Federico and Arroyo, Alvaro and Gu, Xiangming and Perivolaropoulos, Christos and Velickovic, Petar and Pascanu, Razvan and Bronstein, Michael M.},
    title     = {Why do LLMs Attend to the First Token?},
    booktitle = {COLM},
    year      = {2025},
}

@InProceedings{an2025systematic,
    author    = {An, Yongqi and Zhao, Xu and Yu, Tao and Tang, Ming and Wang, Jinqiao},
    title     = {Systematic Outliers in Large Language Models},
    booktitle = {ICLR},
    year      = {2025},
}

@InProceedings{su2026unveiling,
    author    = {Su, Zunhai and Li, Qingyuan and Zhang, Hao and Ye, Weihao and Xue, Qibo and Qian, Yulei and Wong, Ngai and Yuan, Kehong},
    title     = {Unveiling Super Experts in Mixture-of-Experts Large Language Models},
    booktitle = {ICLR},
    year      = {2026},
}

@Article{gallegofeliciano2026hiddendynamics,
    author  = {Gallego-Feliciano, Jorge and McClendon, S. Aaron and Morinelli, Juan and Zervoudakis, Stavros and Saravanos, Antonios},
    title   = {Hidden Dynamics of Massive Activations in Transformer Training},
    journal = {arXiv:2508.03616},
    year    = {2025},
}

@InProceedings{liu2026sinktrack,
    author    = {Liu, Xu and Chen, Guikun and Wang, Wenguan},
    title     = {SinkTrack: Attention Sink Based Context Anchoring for Large Language Models},
    booktitle = {ICLR},
    year      = {2026},
}

@InProceedings{liu2025all,
    author    = {Liu, Yiyang and Liang, James Chenhao and Fan, Heng and Yang, Wenhao and Cui, Yiming and Han, Xiaotian and Huang, Lifu and Liu, Dongfang and Wang, Qifan and Han, Cheng},
    title     = {All You Need is One: Capsule Prompt Tuning with a Single Vector},
    booktitle = {NeurIPS},
    year      = {2025},
}

@InProceedings{fu2026attention,
    title     = {Attention Sink Forges Native MoE in Attention Layers: Sink-Aware Training to Address Head Collapse},
    author    = {Zizhuo Fu and Wenxuan Zeng and Runsheng Wang and Meng Li},
    booktitle = {ICML},
    year      = {2026},
}

@InProceedings{liu2026surgery,
    author    = {Guozhi Liu and Weiwei Lin and Tiansheng Huang and Ruichao Mo and Qi Mu and Xiumin Wang and Li Shen},
    title     = {Surgery: Mitigating Harmful  Fine-Tuning for Large Language Models via Attention Sink},
    booktitle = {ICML},
    year      = {2026},
}

@InProceedings{kaul2025from,
    author    = {Kaul, Prannay and Ma, Chengcheng and Elezi, Ismail and Deng, Jiankang},
    title     = {From Attention to Activation: Unraveling the Enigmas of Large Language Models},
    booktitle = {ICLR},
    year      = {2025},
}

@Misc{miller2023attention,
    author = {Miller, Evan},
    title  = {Attention Is Off By One},
    year   = {2023},
    url    = {https://www.evanmiller.org/attention-is-off-by-one.html},
}

@Article{bu2025valuestategated,
    author  = {Bu, Rui and Zhong, Haofeng and Chen, Wenzheng and Li, Yangyan},
    title   = {Value-State Gated Attention for Mitigating Extreme-Token Phenomena in Transformers},
    journal = {arXiv:2510.09017},
    year    = {2025},
}

@InProceedings{ainslie2023gqa,
    author    = {Ainslie, Joshua and Lee-Thorp, James and de Jong, Michiel and Zemlyanskiy, Yury and Lebron, Federico and Sanghai, Sumit},
    title     = {GQA: Training Generalized Multi-Query Transformer Models from Multi-Head Checkpoints},
    booktitle = {EMNLP},
    year      = {2023},
}

@Article{raffel2020exploring,
    author  = {Raffel, Colin and Shazeer, Noam and Roberts, Adam and Lee, Katherine and Narang, Sharan and Matena, Michael and Zhou, Yanqi and Li, Wei and Liu, Peter J.},
    title   = {Exploring the Limits of Transfer Learning with a Unified Text-to-Text Transformer},
    journal = {JMLR},
    year    = {2020},
}

@InProceedings{liu2020understanding,
    author    = {Liu, Liyuan and Liu, Xiaodong and Gao, Jianfeng and Chen, Weizhu and Han, Jiawei},
    title     = {Understanding the Difficulty of Training Transformers},
    booktitle = {EMNLP},
    year      = {2020},
}

@InProceedings{xiao2023smoothquant,
    title     = {{S}mooth{Q}uant: Accurate and Efficient Post-Training Quantization for Large Language Models},
    author    = {Xiao, Guangxuan and Lin, Ji and Seznec, Mickael and Wu, Hao and Demouth, Julien and Han, Song},
    booktitle = {ICML},
    year      = {2023},
}

@inproceedings{dettmers2023qlora,
    title     = {{QL}o{RA}: Efficient Finetuning of Quantized {LLM}s},
    author    = {Tim Dettmers and Artidoro Pagnoni and Ari Holtzman and Luke Zettlemoyer},
    booktitle = {NeurIPS},
    year      = {2023},
}

@misc{eval-harness,
    author    = {Gao, Leo and Tow, Jonathan and Abbasi, Baber and Biderman, Stella and Black, Sid and DiPofi, Anthony and Foster, Charles and Golding, Laurence and Hsu, Jeffrey and Le Noac'h, Alain and others},
    title     = {A framework for few-shot language model evaluation},
    year      = {2023},
    publisher = {Zenodo},
}

@inproceedings{lin2024awq,
    author    = {Lin, Ji and Tang, Jiaming and Tang, Haotian and Yang, Shang and Chen, Wei-Ming and Wang, Wei-Chen and Xiao, Guangxuan and Dang, Xingyu and Gan, Chuang and Han, Song},
    booktitle = {MLSys},
    title     = {AWQ: Activation-aware Weight Quantization for On-Device LLM Compression and Acceleration},
    year      = {2024},
}

@inproceedings{frantar2023optq,
    title     = {{OPTQ}: Accurate Quantization for Generative Pre-trained Transformers},
    author    = {Elias Frantar and Saleh Ashkboos and Torsten Hoefler and Dan Alistarh},
    year      = {2023},
    booktitle = {ICLR},
}

@article{ba2016layernorm,
    title   = {Layer Normalization},
    author  = {Jimmy Lei Ba and Jamie Ryan Kiros and Geoffrey E. Hinton},
    year    = {2016},
    journal = {arXiv:1607.06450},
}

@inproceedings{dettmers2022gptint,
    title     = {{GPT}3.int8(): 8-bit Matrix Multiplication for Transformers at Scale},
    author    = {Tim Dettmers and Mike Lewis and Younes Belkada and Luke Zettlemoyer},
    booktitle = {NeurIPS},
    year      = {2022},
}

@inproceedings{peng2024yarn,
    title     = {Ya{RN}: Efficient Context Window Extension of Large Language Models},
    author    = {Bowen Peng and Jeffrey Quesnelle and Honglu Fan and Enrico Shippole},
    booktitle = {ICLR},
    year      = {2024},
}

@article{qwen2.5,
    title   = {Qwen2.5 Technical Report},
    author  = {An Yang and Baosong Yang and Beichen Zhang and Binyuan Hui and Bo Zheng and Bowen Yu and Chengyuan Li and Dayiheng Liu and Fei Huang and Haoran Wei and others},
    year    = {2024},
    journal = {arXiv:2412.15115},
}

@article{qwen3,
    title   = {Qwen3 Technical Report},
    author  = {Yang, An and Li, Anfeng and Yang, Baosong and Zhang, Beichen and Hui, Binyuan and Zheng, Bo and Yu, Bowen and Gao, Chang and Huang, Chengen and Lv, Chenxu and others},
    journal = {arXiv:2505.09388},
    year    = {2025},
}

@article{glm2024chatglm,
    title   = {ChatGLM: A Family of Large Language Models from GLM-130B to GLM-4 All Tools},
    author  = {Aohan Zeng and Bin Xu and Bowen Wang and Chenhui Zhang and Da Yin and Dan Zhang and Diego Rojas and Guanyu Feng and Hanlin Zhao and Hanyu Lai and others},
    year    = {2024},
    journal = {arXiv:2406.12793},
}

@article{grattafiori2024llama3herdmodels,
    title   = {The Llama 3 Herd of Models},
    author  = {Grattafiori, Aaron and Dubey, Abhimanyu and Jauhri, Abhinav and Pandey, Abhinav and Kadian, Abhishek and Al-Dahle, Ahmad and Letman, Aiesha and Mathur, Akhil and Schelten, Alan and Vaughan, Alex and others},
    year    = {2024},
    journal = {arXiv:2407.21783},
}

@article{radford2019language,
    title  = {Language Models are Unsupervised Multitask Learners},
    author = {Radford, Alec and Wu, Jeff and Child, Rewon and Luan, David and Amodei, Dario and Sutskever, Ilya},
    year   = {2019},
}

@InProceedings{qiu2026flash,
    author    = {Qiu, Haiquan and Yao, Quanming},
    title     = {Why Low-Precision Transformer Training Fails: An Analysis on {Flash Attention}},
    booktitle = {ICLR},
    year      = {2026},
}

@InProceedings{wu2025icl,
    author    = {Wu, Shiguang and Wang, Yaqing and Yao, Quanming},
    title     = {Why In-Context Learning Models are Good Few-Shot Learners?},
    booktitle = {ICLR},
    year      = {2025},
}

@InProceedings{wang2026deal,
    author    = {Wang, Yaqing and Luo, Zhenlin and Zhao, Peiyao and Cai, Yunfeng and Yao, Quanming},
    title     = {Beyond Scaling: A Survey on Data-Efficient Agentic Learning},
    booktitle = {IJCAI},
    year      = {2026},
}

\clearpage
\appendix

\section{Related work}
\label{app:related_work}

\paragraph{Phenomena and utilization}
Attention sinks (AS) refer to the tendency of Transformer-based LLMs to route disproportionate attention mass to a small set of early or otherwise uninformative tokens~\cite{xiao2024streamingllm,gu2025when}. Their practical importance was first made especially clear by StreamingLLM, which showed that preserving sink tokens in the KV cache enables stable streaming inference beyond the training context window~\cite{xiao2024streamingllm}. Subsequent work has used or analyzed sink tokens for KV-cache compression, long-context inference, context anchoring, prompt tuning, and fine-tuning robustness~\cite{su2025kvsink,liu2026sinktrack,liu2025all,fu2026attention,liu2026surgery}. 
Massive activations (MA) are sparse but extremely large activation values that recur in LLMs and often appear at specific token positions or features~\cite{sun2024massive,an2025systematic}. They are closely related to activation outliers that dominate numerical ranges and complicate low-bit inference~\cite{dettmers2022gptint,bondarenko2023quantizable,xiao2023smoothquant}.
AS and MA are extreme-token phenomena whose functional roles and training origins have become active mechanistic questions~\cite{guo2025activedormant,qiu2026unified}. Our work identifies gradient sinks as a missing backward-pass phenomenon.

\paragraph{Mechanistic interpretations}
The study of extreme-token phenomena has produced several complementary explanations. 
One influential view interprets attention sinks as a way for attention heads to implement ``no-op'' behavior under the sum-to-one constraint of softmax attention. When a head has little useful information to mix into the current representation, it can place most of its attention on low-information tokens whose value vectors contribute little to the residual update~\cite{bondarenko2023quantizable,gu2025when}. 
Other work argues that first-token attention can serve useful information-propagation roles, such as reducing over-mixing in deep Transformers~\cite{barbero2025why}. 
Closely related work studies value-state drains, where sink tokens receive high attention but have unusually small value norms, and active-dormant head dynamics and value-state gating both emphasize this attention-value coupling~\cite{guo2025activedormant,bu2025valuestategated}. 
Another line connects attention sinks to massive activations, compression valleys, and outlier-driven rescaling, explaining how extreme residual or activation structures can shape forward attention behavior~\cite{sun2024massive,kaul2025from,queipo-de-llano2026attention,sun2026anatomy,qiu2026unified}. 
These accounts clarify why AS and MA often co-occur and why they can be useful or stable forward structures.
Our work complements these explanations by treating attention sinks as sources of backward gradient concentration.

\paragraph{Mitigation and intervention}
Many studies regard attention sinks or activation outliers as harmful to LLMs and seek to mitigate them.
One family changes the softmax mechanism itself, using clipped softmax, rectified alternatives, or related normalizers to reduce forced attention allocation and activation outliers~\cite{bondarenko2023quantizable,miller2023attention,kaul2025from,zuhri2025softpick}. 
Another family introduces gates after the attention output or on the value states, giving heads an explicit route for suppressing no-op updates without relying on extreme attention logits~\cite{qiu2025gated,bu2025valuestategated}. 
Recent work further studies the anatomy of AS and MA across architectures, showing that architectural choices can decouple the two phenomena~\cite{sun2026anatomy}. 
Our proposed V-scale has a narrower purpose. It is designed as a mechanism test that keeps query-key attention scores unchanged while attenuating value-path gradient pressure, allowing us to ask whether MA can be suppressed while AS is largely preserved.

\paragraph{Normalization and pre-norm Transformers}
As a central component in modern Transformers, layer normalization was introduced to stabilize hidden-state dynamics, and subsequent work showed that normalization placement in Transformer residual blocks strongly affects optimization and gradient transport~\cite{ba2016layernorm,xiong2020layernorm}. 
RMSNorm removes mean-centering while preserving rescaling invariance, and it has become standard in Llama-like decoder-only models~\cite{zhang2019rmsnorm,touvron2023llama,grattafiori2024llama3herdmodels,qwen2.5,qwen3,glm2024chatglm}.
In modern pre-norm RMSNorm Transformers, attention and MLP branches operate on normalized inputs, so very large residual-stream norms have a limited direct role in the forward branch computation. This naturally motivates us to ask how normalization shapes backward signal propagation during training.
Qiu and Yao~\cite{qiu2026flash} study how numerical errors in low-precision FlashAttention accumulate through training dynamics and lead to instability. Our work examines how attention-induced gradient concentration is regulated through activation scale in pre-norm Transformers.

\clearpage

\section{Model and training configuration}\label{app:config}

Table~\ref{tab:config} summarizes the configurations for the scratch-trained baselines and their matched V-scale variants. Unless otherwise specified, all models are decoder-only Transformers with RMSNorm, RoPE positional encoding, SwiGLU feed-forward blocks, and causal self-attention. We use the GPT-2 tokenizer~\cite{radford2019language}. All training runs were implemented in PyTorch using NVIDIA A100 GPUs (80GB). During preprocessing, documents are packed by inserting \texttt{<eos>} as the separator, without introducing dedicated \texttt{<bos>}. Weight decay is applied only to 2-D weight matrices. V-scale runs use the same data, optimizer, and schedule as their corresponding baselines.

\begin{table}[ht]
	\centering
	\caption{Model and training configuration.}\label{tab:config}
	\begin{tabular}{p{0.25\linewidth}p{0.15\linewidth}p{0.15\linewidth}p{0.15\linewidth}}
		\toprule
		Item & 0.1B runs & 0.3B runs & 1B runs \\
		\midrule
		\emph{Model} & & & \\
		total parameters & $\sim$101M & $\sim$304M & $\sim$1.04B \\
		vocab size & 50257 & 50257 & 50257 \\
		hidden layers & 16 & 17 & 18 \\
		hidden size & 512 & 1024 & 2048 \\
		attention heads & 8 & 8 & 16 \\
		key/value heads & 8 & 4 & 8 \\
		head dimensions & 64 & 128 & 128 \\
		intermediate size & 1408 & 2816 & 5504 \\
		RMSNorm $\epsilon$ & 1e-5 & 1e-5 & 1e-5 \\
		RoPE base & 10000 & 10000 & 10000 \\
		\midrule
		\emph{Optimizer} & & & \\
		optimizer & AdamW & AdamW & AdamW \\
		AdamW $\beta_1$ & 0.9 & 0.9 & 0.9 \\
		AdamW $\beta_2$ & 0.95 & 0.95 & 0.95 \\
		AdamW $\epsilon$ & 1e-8 & 1e-8 & 1e-8 \\
		weight decay & 0.1 & 0.1 & 0.1 \\
		gradient clipping norm & 1.0 & 1.0 & 1.0 \\
		max learning rate & 2e-3 & 1e-3 & 1e-3 \\
		min learning rate & 2e-4 & 1e-4 & 1e-4 \\
		\midrule
		\emph{Training} & & & \\
		learning rate schedule & cosine decay & cosine decay & cosine decay \\
		total steps & 20000 & 20000 & 25000 \\
		warmup steps & 1000 & 1000 & 1000 \\
		context length & 1024 & 2048 & 2048 \\
		global batch size & 512 & 512 & 512 \\
		tokens per step & 0.52M & 1.05M & 1.05M \\
		total tokens & 10.49B & 20.97B & 26.21B \\
		training data & C4 & C4 & C4 \\
		\bottomrule
	\end{tabular}
\end{table}

\clearpage

\section{Additional empirical results}\label{app:empirical}

\subsection{Supplementary results for QKV gradients}

This subsection supplements the QKV-path decomposition in Figure~\ref{fig:qkv_grads_over_tokens}.
Figure~\ref{fig:qkv_grads_scratch} extends the diagnostic to smaller scratch-trained models, using the same layout as the 1B result in the main text.
Figure~\ref{fig:qkv_grads_pretrained} reports the corresponding measurements for pretrained LLMs.
Across these settings, the query pathway remains comparatively flat, whereas key and especially value gradients show a pronounced spike at token 0.

\begin{figure}[ht]
	\centering
	\includegraphics[width=0.3\textwidth]{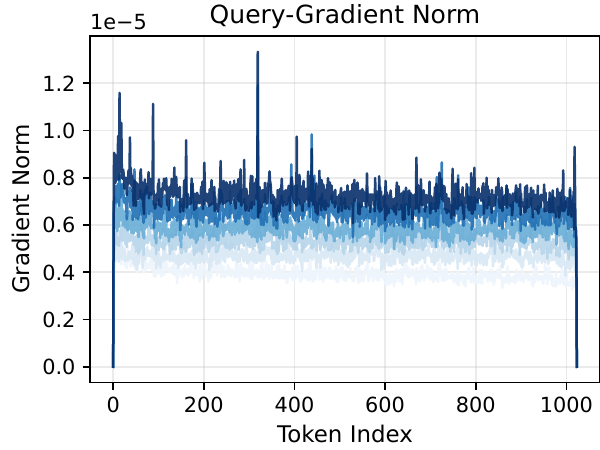}
	\includegraphics[width=0.3\textwidth]{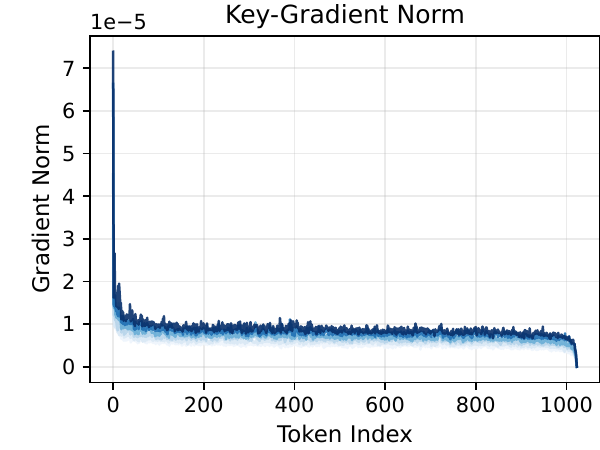}
	\includegraphics[width=0.3\textwidth]{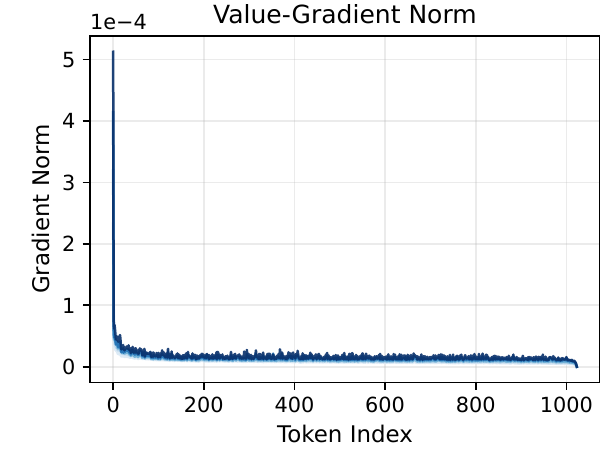}
	\includegraphics[width=0.06\textwidth]{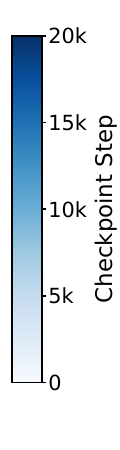} \\
	\includegraphics[width=0.3\textwidth]{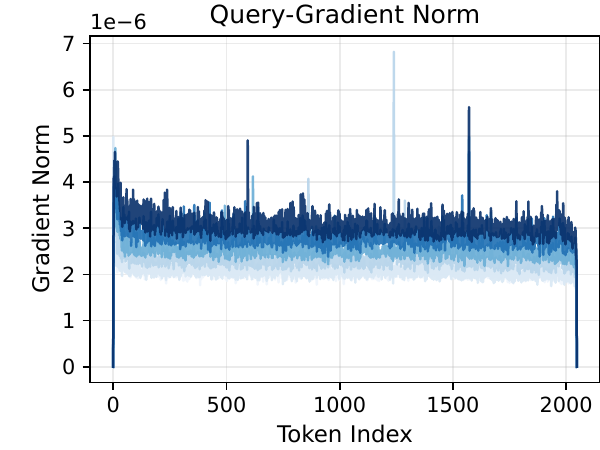}
	\includegraphics[width=0.3\textwidth]{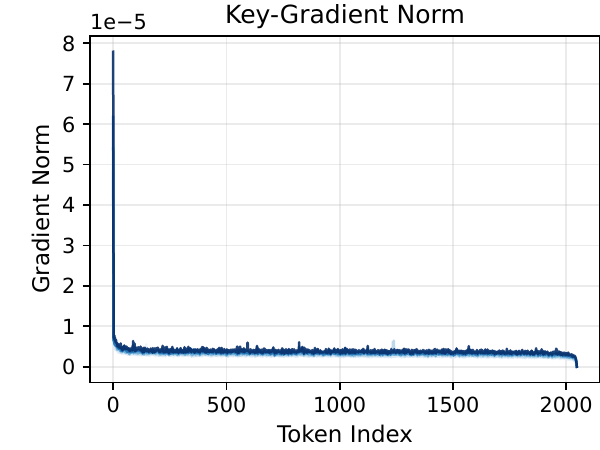}
	\includegraphics[width=0.3\textwidth]{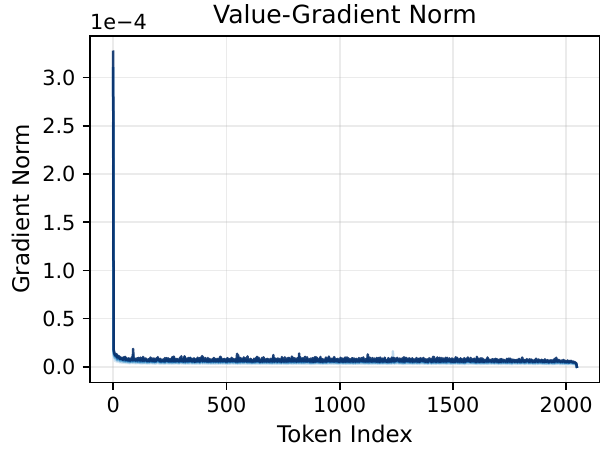}
	\includegraphics[width=0.06\textwidth]{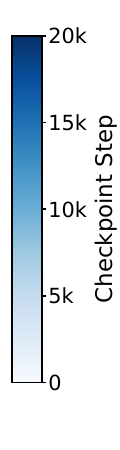}
	
    \caption{
		Token-wise QKV gradient norms across training checkpoints for smaller scratch-trained models, averaged over layers.
		Top row: 0.1B model; bottom row: 0.3B model.
		From left to right: query, key, and value gradient norms as functions of token position.
		Both models reproduce the pathway asymmetry observed in the 1B model: key and especially value gradients exhibit a strong spike at token 0, while query gradients remain comparatively flat.
	}\label{fig:qkv_grads_scratch}
\end{figure}

\begin{figure}[ht]
	\centering
	\setlength{\tabcolsep}{0pt}
	\newcommand{\gradrow}[2]{%
		\rotatebox[origin=l]{90}{\quad #1} &
		\includegraphics[width=0.3\textwidth]{pics/#2/LINE_Dq_over_tokens.pdf} &
		\includegraphics[width=0.3\textwidth]{pics/#2/LINE_Dk_over_tokens.pdf} &
		\includegraphics[width=0.3\textwidth]{pics/#2/LINE_Dv_over_tokens.pdf}
		\\
	}
	\begin{tabular}{cccc}
		\gradrow{Qwen2.5-7B}{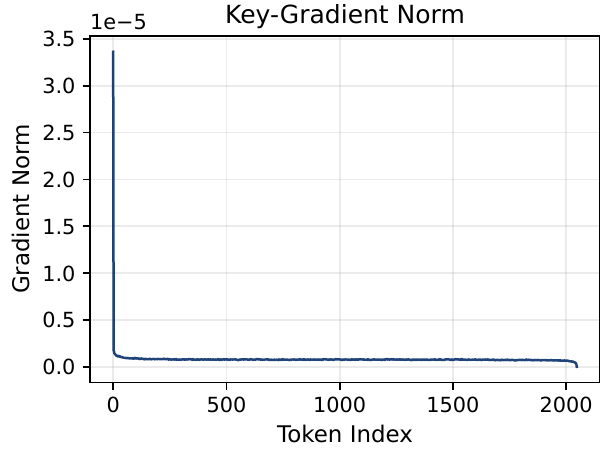}
		\gradrow{Qwen3-8B}{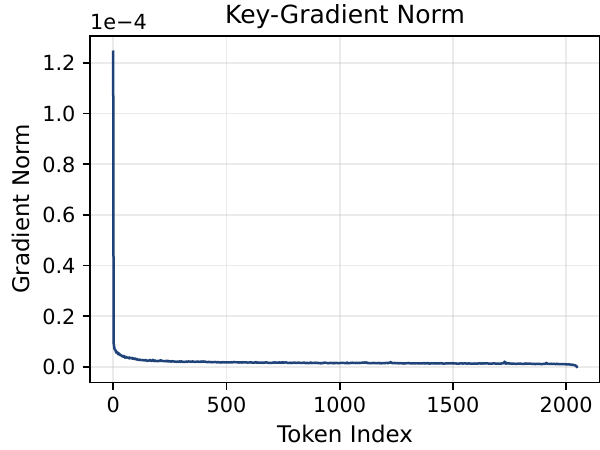}
		\gradrow{Llama-3.1-8B}{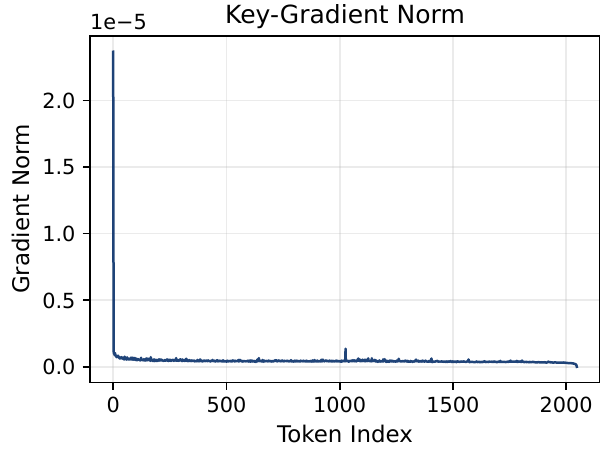}
		\gradrow{GLM-4-9B}{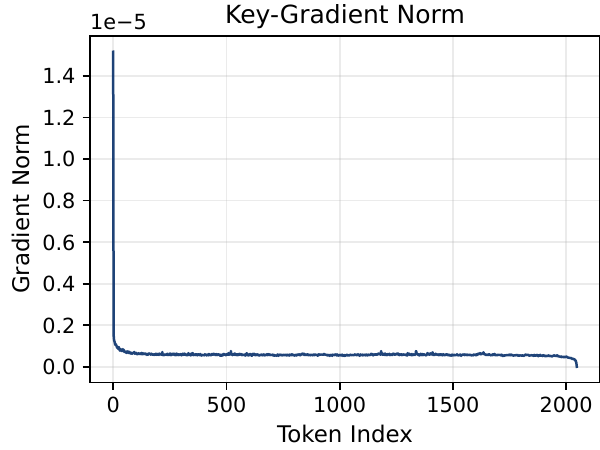}
	\end{tabular}

	\caption{
		Token-wise gradient norms of QKV for pretrained LLMs.
		Each row corresponds to one model.
		From left to right: gradient norms of query, key, and value as functions of token position, averaged over layers.
		Across models, key and especially value gradients consistently exhibit a pronounced spike at token 0.
	}\label{fig:qkv_grads_pretrained}
\end{figure}

\subsection{Supplementary results for gradient reshaping}\label{app:mlp_reshaping}

This subsection complements the gradient-reshaping measurements in Section~\ref{sec:ma_grad_reshape} with the analogous MLP-branch diagnostics.
Unlike the attention-branch results in Figure~\ref{fig:attn_scatter_main}, the left panel of Figure~\ref{fig:mlp_scatter_main} shows that $\mathrm{Bloat}^{\mathrm{mlp}}$ is much less dominated by token 0.
Thus, the MLP branch does not appear to be the primary location where localized pressure is generated.
The middle panel shows that $\mathrm{Change}^{\mathrm{mlp}}$ remains concentrated near zero ($\log_{10}1$), and the right panel again shows a tight alignment between activation norms and compression, as observed in attention blocks.
Figure~\ref{fig:reshaping_supplement} gives the corresponding attention-branch and MLP-branch measurements for the 0.3B baseline and pretrained LLMs.
The same qualitative contrast persists: bloat at token 0 is sharper in the attention branch than in the MLP branch, whereas compression remains strongly tied to activation scale in both branches.
Qwen3-8B is somewhat less aligned with the other pretrained models, which may reflect its post-training procedure.

\begin{figure}[tb]
    \centering
    \includegraphics[width=0.3\linewidth]{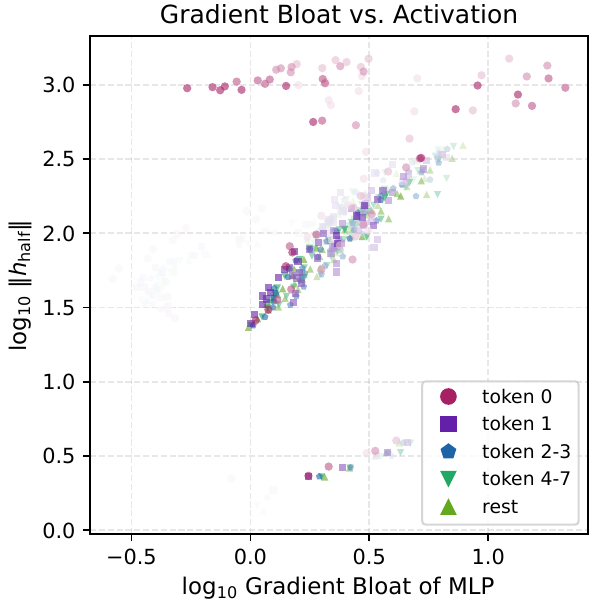}
    \includegraphics[width=0.3\linewidth]{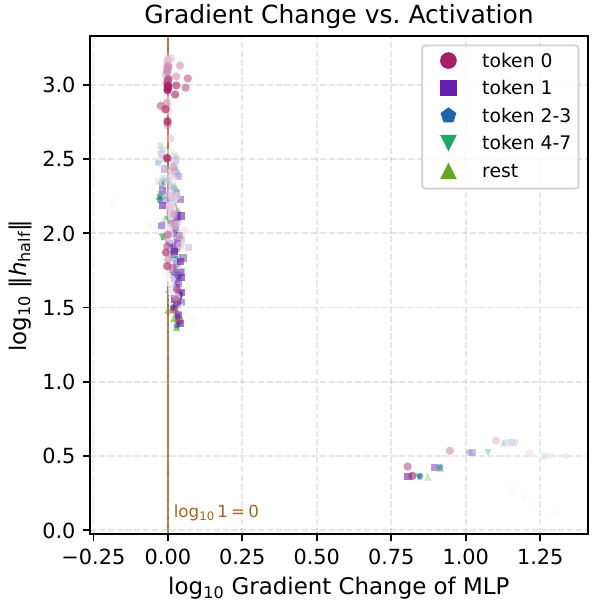}
    \includegraphics[width=0.3\linewidth]{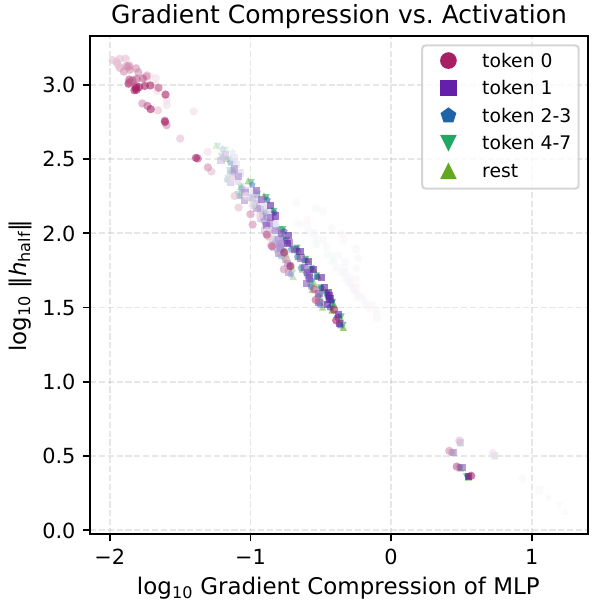}
    \includegraphics[width=0.06\linewidth]{pics/01B_baseline/scatter_group_steps.pdf}\\
    \includegraphics[width=0.3\linewidth]{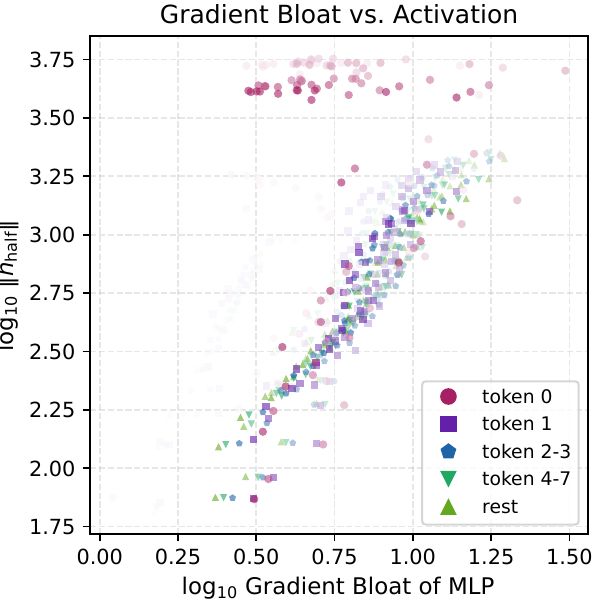}
    \includegraphics[width=0.3\linewidth]{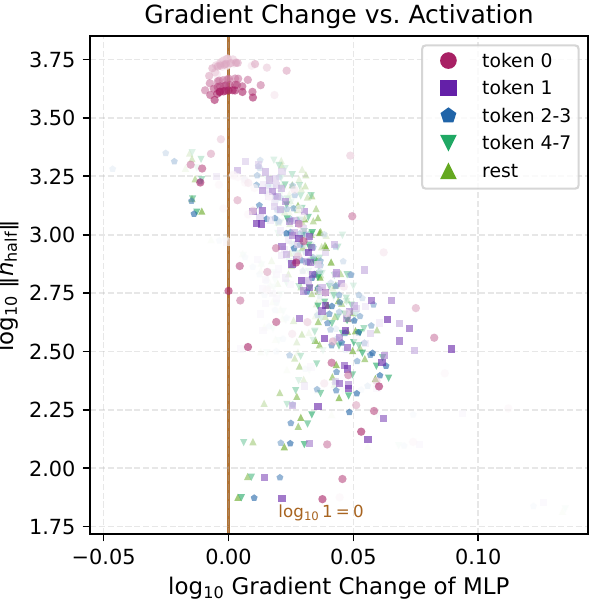}
    \includegraphics[width=0.3\linewidth]{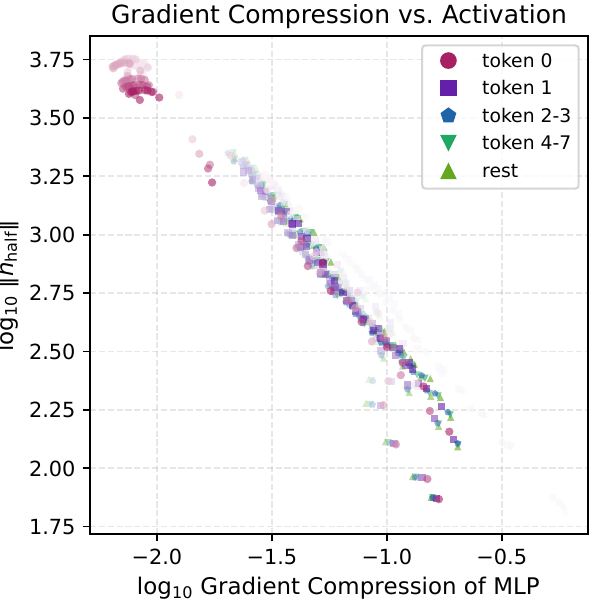}
    \includegraphics[width=0.06\linewidth]{pics/1B_baseline/scatter_group_steps.pdf}

    \caption{
    Scatter plots relating gradient reshaping to input activation norms of the MLP block. Top row: 0.1B model; bottom row: 1B model.
    From left to right: $\log_{10}\mathrm{Bloat}^{\mathrm{mlp}}$ vs.\ $\log_{10}\|h_{\mathrm{half}}\|$, 
    $\log_{10}\mathrm{Change}^{\mathrm{mlp}}$ vs.\ $\log_{10}\|h_{\mathrm{half}}\|$, and
    $\log_{10}\mathrm{Compress}^{\mathrm{mlp}}$ vs.\ $\log_{10}\|h_{\mathrm{half}}\|$.
    Each point is colored by token group (token 0, token 1, tokens 2--3, tokens 4--7, and the rest early tokens 8--15).
    Compared with the attention branch, token 0 is less dominant in MLP bloat, while compression again aligns tightly with activation scale and the $\mathrm{Change}$ ratio remains close to one.
    }\label{fig:mlp_scatter_main}
\end{figure}

\begin{figure}[tb]
    \centering
    \setlength{\tabcolsep}{0pt}
    \newcommand{\gradrow}[2]{%
        \rotatebox[origin=l]{90}{ #1} &
        \includegraphics[width=0.163\textwidth]{pics/#2/Ea3_Bloat.pdf} &
        \includegraphics[width=0.163\textwidth]{pics/#2/Ea4_Change.pdf} &
        \includegraphics[width=0.163\textwidth]{pics/#2/Ea1_compress.pdf} &
        \includegraphics[width=0.163\textwidth]{pics/#2/Em3_Bloat.pdf} &
        \includegraphics[width=0.163\textwidth]{pics/#2/Em4_Change.pdf} &
        \includegraphics[width=0.163\textwidth]{pics/#2/Em1_compress.pdf}
        \\
    }
    \begin{tabular}{ccccccc}
        \gradrow{0.3B Baseline}{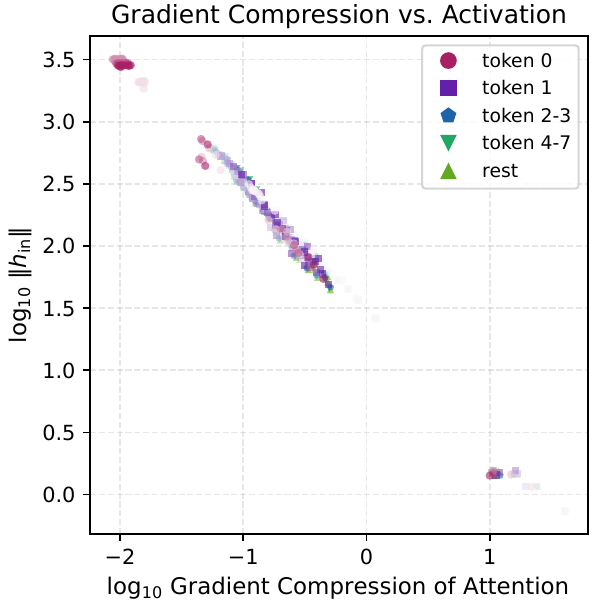}
        \gradrow{Qwen2.5-7B}{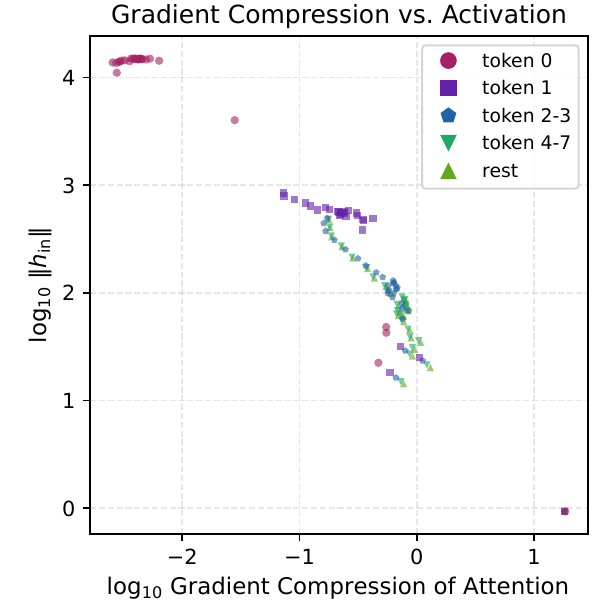}
        \gradrow{Qwen3-8B}{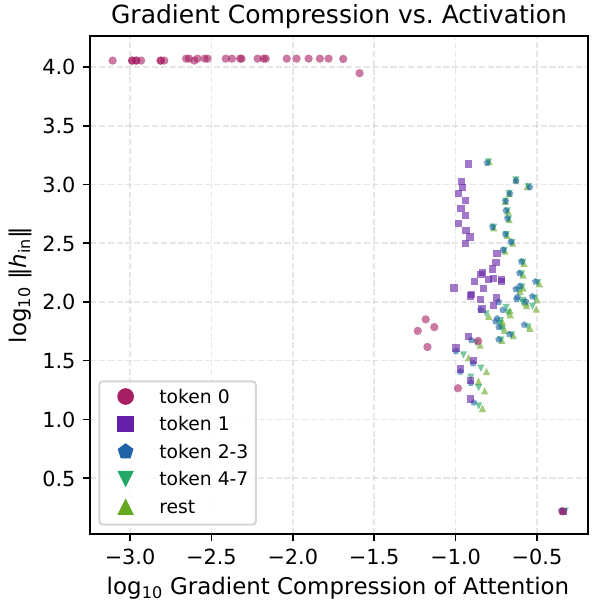}
        \gradrow{Llama-3.1-8B}{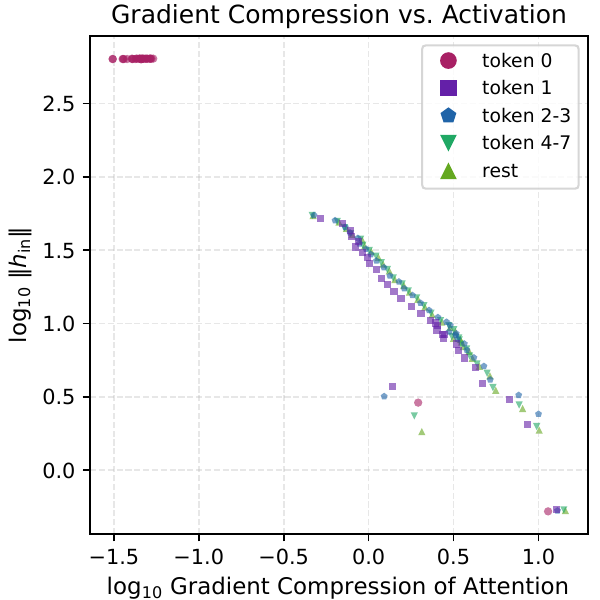}
        \gradrow{GLM-4-9B}{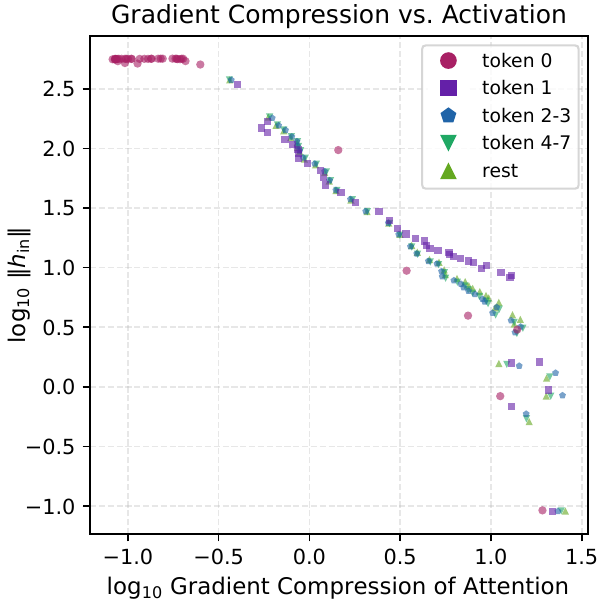}
    \end{tabular}
    
    \caption{
    Supplementary gradient-reshaping measurements for the 0.3B baseline and pretrained models. From left to right, each row shows attention-branch $\mathrm{Bloat}$, $\mathrm{Change}$, and $\mathrm{Compress}$, followed by MLP-branch $\mathrm{Bloat}$, $\mathrm{Change}$, and $\mathrm{Compress}$.
    }\label{fig:reshaping_supplement}
\end{figure}

\subsection{Supporting observations for V-scale design}
\label{app:support_observation}

This subsection provides additional diagnostics for the design of V-scale.
This intervention is intended to act as a value-path gradient valve. It should be strongest when the sink token has a small value-state norm, and the learned scale parameter should adapt across layers and heads.
Figure~\ref{fig:value_state_drain} checks the first condition in baseline models.
Across training checkpoints, the value state of token 0 is usually much smaller than the mean and maximum over other early tokens, especially in middle and later layers.
Since the V-scale Jacobian attenuates gradients most strongly when $\|v\|_2^2 \ll C_{\ell,h}$, this pattern makes the intervention selective for the sink token rather than a uniform shrinkage of all value states.

\begin{figure}[t]
    \centering
    \includegraphics[width=0.3\linewidth]{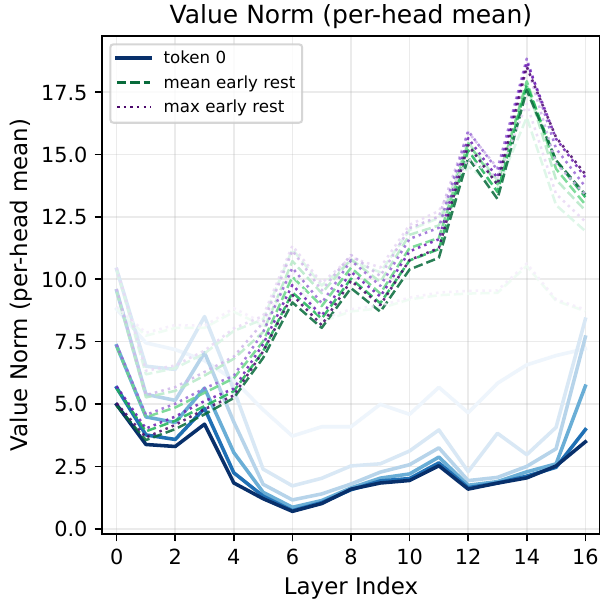}
    \includegraphics[width=0.06\linewidth]{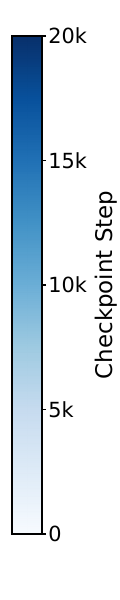}
    \includegraphics[width=0.3\linewidth]{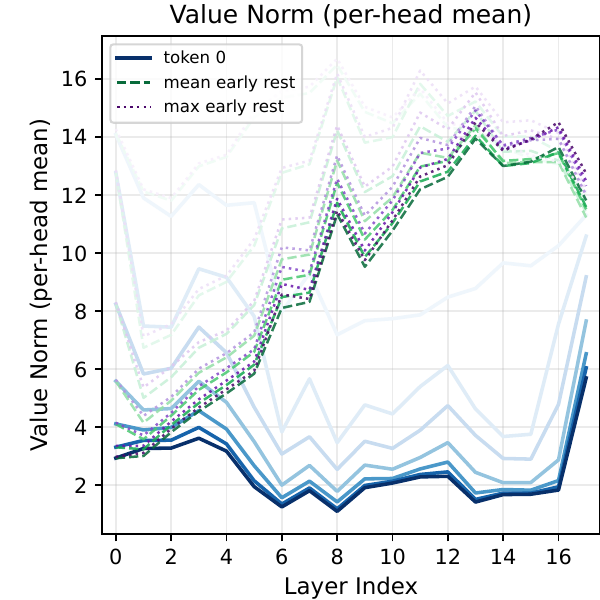}
    \includegraphics[width=0.06\linewidth]{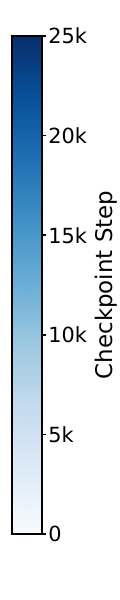}
    \caption{
    Value-state norms in baseline models, averaged over heads and shown across training checkpoints.
    Left: 0.3B baseline; right: 1B baseline.
    The solid curve is token 0, the dashed curve is the mean over other early tokens, and the dotted curve is their maximum, where ``early'' refers to positions 1--15.
    Token 0 usually has a much smaller value-state norm, placing it in the regime where V-scale can induce stronger value-path gradient attenuation.
    }\label{fig:value_state_drain}
\end{figure}

Figure~\ref{fig:vscale_theta_heatmap} reports the final learned V-scale parameters.
Since $\lambda^{\ell,h}=\exp(\theta_{\ell,h})$ is initialized to one, positive $\theta_{\ell,h}$ increases the value-state scale parameter $C_{\ell,h}$ and strengthens attenuation for a fixed small value norm, while negative $\theta_{\ell,h}$ decreases $C_{\ell,h}$.
Both scales show a clear layer-dependent pattern, with larger positive values concentrated in earlier layers and mostly negative values in later layers.
This indicates that V-scale is not uniform but instead learns layer- and head-specific strength for the value-path valve.
Such heterogeneity may interact with specialized pathways for selecting relevant information under competing cues, potentially connecting to the gains on multi-key NIAH, BoolQ, and ARC-C reported in Appendix~\ref{app:additional_downstream}. 
This hypothesis requires further evidence.

\begin{figure}[tb]
    \centering
    \includegraphics[width=0.4\linewidth]{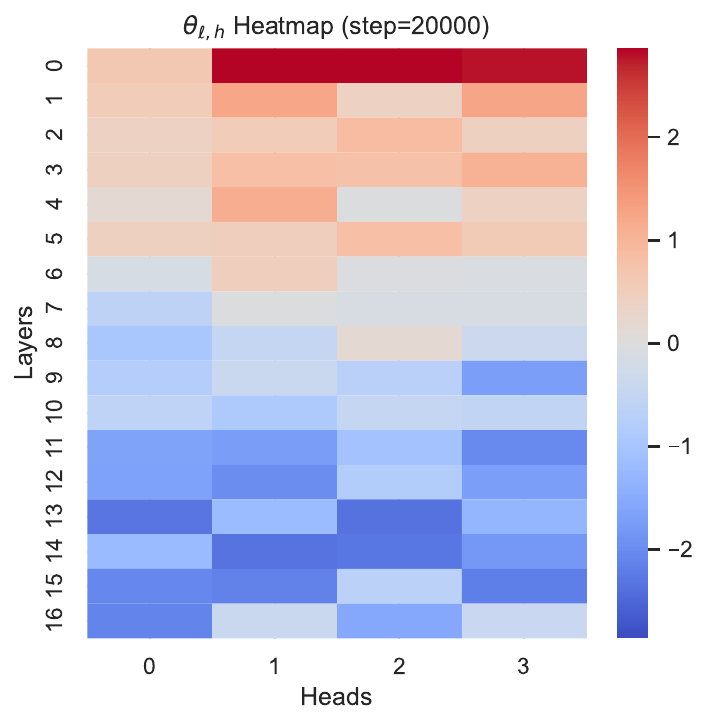}
    \includegraphics[width=0.4\linewidth]{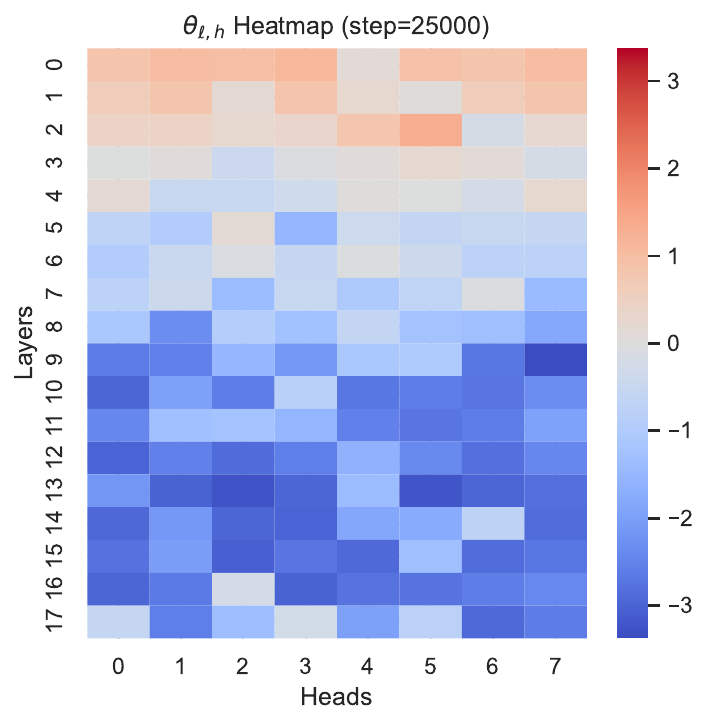}
    \caption{
    Final learned V-scale parameters $\theta_{\ell,h}$.
    Left: 0.3B V-scale model; right: 1B V-scale model.
    Each heatmap indexes layers by rows and attention heads by columns, with the color scale centered at the initialization value $\theta=0$.
    Positive values correspond to larger $C_{\ell,h}$ and stronger attenuation for fixed small-norm value states, while negative values correspond to smaller $C_{\ell,h}$.
    }\label{fig:vscale_theta_heatmap}
\end{figure}

Finally, Figure~\ref{fig:vscale_gs_supplement} complements the 0.3B comparison in Figure~\ref{fig:vscale-gs} by showing the distributions of gradient sinks for additional V-scale models.
Token 0 remains the largest-gradient position but the remaining concentration is moderate compared to the matched baselines in Figure~\ref{fig:gs-definition}.

\begin{figure}[t]
    \centering
    \includegraphics[width=0.48\textwidth]{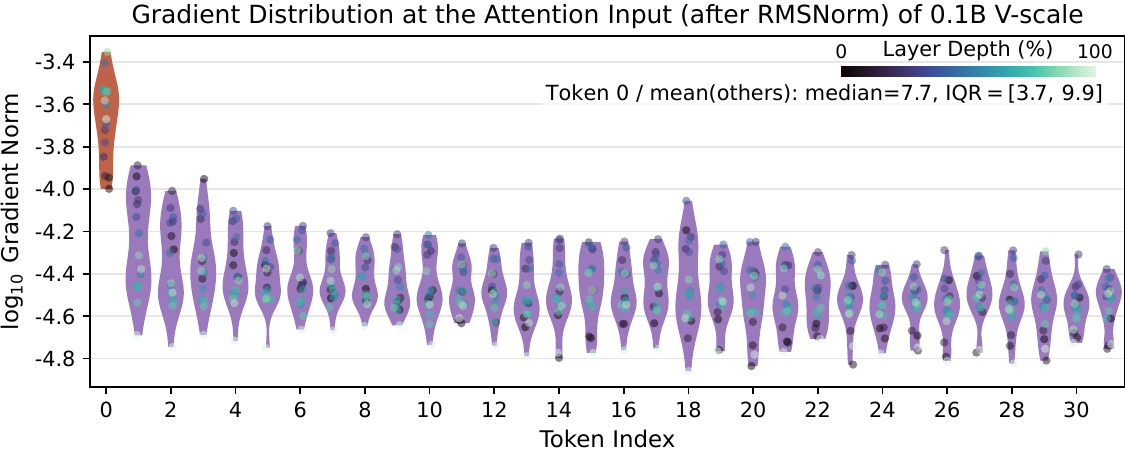}
    \includegraphics[width=0.48\textwidth]{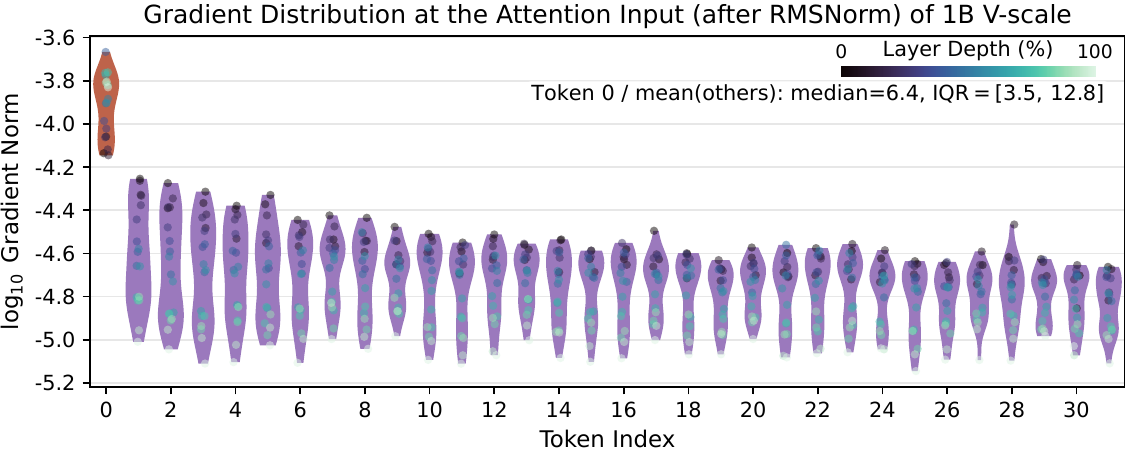}
    \caption{
    Gradient sinks in additional V-scale models.
    Left: 0.1B V-scale model; right: 1B V-scale model.
    Each panel plots the token-wise distribution of $\log_{10}\|\nabla_{\widetilde h_t^{\ell}}\mathcal L\|_2$ at the post-RMSNorm attention input, with token 0 highlighted.
    The sink token remains visible but is more moderate.
    }\label{fig:vscale_gs_supplement}
\end{figure}

\subsection{Additional downstream tasks}\label{app:additional_downstream}

This subsection reports additional downstream evaluations, supplementing the compact summary in Section~\ref{sec:vscale}.
We focus on 1B models because this is the largest scratch-trained setting in our study.
We pair each baseline with its V-scale counterpart and evaluate language modeling, reasoning, and retrieval in \texttt{bfloat16} and under post-training quantization. These measurements assess how the intervention affects downstream capability across precision settings.

Table~\ref{tab:additional_downstream_1b} shows that V-scale maintains similar perplexity and improves ARC-C, BoolQ, COPA, and OpenBookQA under every listed precision setting. Several other tasks show small decreases. Together with the consistent multi-key retrieval gains below, these results show that suppressing MA preserves overall language-model capability and benefits several tasks requiring selective information use.

\begin{table}[ht]
    \centering
    \caption{
    Standard downstream evaluation of 1B baseline and V-scale models with unquantized \texttt{bfloat16} and post-training quantization.
    \texttt{bfloat16} denotes unquantized evaluation. SmoothQuant is W8A8 quantization, and BNB, GPTQ, and AWQ are W4A16 quantization methods.
    Wiki ppl and Lambada ppl are perplexities, where lower is better.
    Lambada acc, ARC-C/E (AI2 Reasoning Challenge and Easy splits), BoolQ, COPA, HellaSwag, OBQA (OpenBookQA), PIQA, SCIQ, and WinoGrande are zero-shot accuracies in percent, where higher is better.
    V-scale keeps perplexity nearly unchanged and remains broadly comparable to the baseline across reasoning tasks, with gains on ARC-C, BoolQ, COPA, and OBQA and small drops on several other tasks.
    Bold marks the better result within each baseline and V-scale pair under the same evaluation setting.
    }\label{tab:additional_downstream_1b}
    \setlength\tabcolsep{3pt}
    \begin{tabular}{llccccccc}
        \toprule
        \multicolumn{9}{l}{\emph{Panel A: language modeling and core reasoning}} \\
        \midrule
        \multirow{2}{*}{Model} & \multirow{2}{*}{Quantization} & Wiki & \multicolumn{2}{c}{Lambada} & ARC-C & ARC-E & BoolQ & COPA \\
        & & ppl$\downarrow$ & ppl$\downarrow$ & acc$\uparrow$ & acc$\uparrow$ & acc$\uparrow$ & acc$\uparrow$ & acc$\uparrow$ \\
        \midrule
        Baseline & bfloat16 & 22.84 & \textbf{15.27} & 43.28 & 24.91 & \textbf{47.98} & 51.38 & 69.00 \\
        V-scale & bfloat16 & \textbf{22.83} & 15.91 & \textbf{43.37} & \textbf{27.05} & 47.47 & \textbf{56.27} & \textbf{70.00} \\
        Baseline & SmoothQuant & 23.17 & 16.61 & 43.02 & 25.26 & \textbf{47.47} & 50.92 & 67.00 \\
        V-scale & SmoothQuant & \textbf{23.01} & \textbf{16.58} & \textbf{43.14} & \textbf{26.11} & 46.76 & \textbf{55.57} & \textbf{69.00} \\
        Baseline & BNB & 23.49 & 16.40 & 42.73 & 25.00 & \textbf{47.43} & 54.19 & 72.00 \\
        V-scale & BNB & \textbf{23.40} & \textbf{15.86} & \textbf{43.55} & \textbf{27.56} & 47.05 & \textbf{55.96} & \textbf{74.00} \\
        Baseline & GPTQ & 23.21 & \textbf{15.70} & 43.26 & 25.77 & \textbf{48.48} & 52.54 & 68.00 \\
        V-scale & GPTQ & \textbf{23.17} & 16.48 & \textbf{43.76} & \textbf{26.79} & 46.80 & \textbf{56.45} & \textbf{70.00} \\
        Baseline & AWQ & \textbf{23.32} & \textbf{16.34} & 42.44 & 25.26 & \textbf{48.06} & 53.64 & 68.00 \\
        V-scale & AWQ & 23.36 & 16.95 & \textbf{43.06} & \textbf{26.71} & 47.10 & \textbf{55.96} & \textbf{69.00} \\
        \bottomrule
    \end{tabular}
    \begin{tabular}{llccccc}
        \toprule
        \multicolumn{7}{l}{\emph{Panel B: additional commonsense and science QA tasks}} \\
        \midrule
        Model & Quantization & HellaSwag & OBQA & PIQA & SCIQ & WinoGrande \\
        \midrule
        Baseline & bfloat16 & \textbf{51.59} & 31.40 & \textbf{72.52} & \textbf{69.90} & \textbf{54.22} \\
        V-scale & bfloat16 & 51.29 & \textbf{32.80} & 71.49 & 69.30 & 53.28 \\
        Baseline & SmoothQuant & \textbf{51.49} & 30.80 & \textbf{72.47} & \textbf{69.00} & \textbf{55.17} \\
        V-scale & SmoothQuant & 51.13 & \textbf{32.40} & 71.00 & 68.30 & 54.46 \\
        Baseline & BNB & \textbf{51.59} & 31.40 & \textbf{73.01} & 69.20 & \textbf{55.56} \\
        V-scale & BNB & 51.29 & \textbf{32.80} & 71.27 & \textbf{69.50} & 54.14 \\
        Baseline & GPTQ & \textbf{51.61} & 31.60 & \textbf{72.85} & \textbf{69.20} & \textbf{54.70} \\
        V-scale & GPTQ & 51.31 & \textbf{32.40} & 71.44 & 68.40 & \textbf{54.70} \\
        Baseline & AWQ & \textbf{51.23} & 30.60 & \textbf{72.47} & \textbf{69.40} & \textbf{54.30} \\
        V-scale & AWQ & 50.88 & \textbf{33.20} & 71.65 & 69.10 & 53.51 \\
        \bottomrule
    \end{tabular}
\end{table}

We next evaluate long-context retrieval using three Needle-in-a-Haystack (NIAH) variants. The native 1B models are evaluated up to 2048 tokens. We apply YaRN with an extension factor of $4$ to evaluate contexts up to 8192 tokens. Tables~\ref{tab:niah_single2_1b}--\ref{tab:niah_multikey1_1b} report means and standard deviations over $5$ evaluation seeds. V-scale consistently improves multi-key retrieval across every listed context length and precision setting in Table~\ref{tab:niah_multikey1_1b}. The single-needle template in Table~\ref{tab:niah_single2_1b} is close to saturated at shorter lengths and also improves under every precision setting at 8192 tokens.

The harder single-needle template in Table~\ref{tab:niah_single3_1b} gives mixed results. V-scale improves native 2048-token accuracy under \texttt{bfloat16}, SmoothQuant, BNB, and GPTQ, while AWQ accuracy falls. After YaRN extension, GPTQ and AWQ accuracy at 4096 tokens falls. These degradations show that the benefit observed in multi-key retrieval depends on the retrieval task. The interaction with W4A16 quantization and context extension remains to be explained.

\begin{table}[ht]
    \centering
    \caption{
    Needle-in-a-Haystack retrieval accuracy on the \texttt{niah\_single\_2} single-needle template for 1B models.
    Columns denote context length in tokens, and entries are mean accuracy percentages with standard deviations over $5$ random seeds; higher is better.
    The first block evaluates native-context 1B models, for which only 1024 and 2048 tokens are tested.
    Blank cells therefore indicate untested longer contexts rather than failed runs.
    The second block evaluates the same models after YaRN extension with a factor of $4$, enabling 4096- and 8192-token tests.
    This template is nearly saturated at shorter lengths, but V-scale gives consistently higher accuracy at 8192 tokens after YaRN extension.
    Bold marks the better result within each baseline and V-scale pair under the same evaluation setting.
    }\label{tab:niah_single2_1b}
    \begin{tabular}{llcccc}
        \toprule
        Model & Quantization & 1024 & 2048 & 4096 & 8192 \\
        \midrule
        \rowcolor[gray]{0.9}\multicolumn{6}{l}{\emph{1B Models}} \\
        Baseline & bfloat16 & $99.96_{\pm 0.09}$ & $\mathbf{98.00}_{\pm 0.69}$ & & \\
        V-scale & bfloat16 & $\mathbf{100.00}_{\pm 0.00}$ & $97.80_{\pm 0.42}$ & & \\
        Baseline & SmoothQuant & $99.96_{\pm 0.09}$ & $\mathbf{95.92}_{\pm 2.09}$ & & \\
        V-scale & SmoothQuant & $\mathbf{100.00}_{\pm 0.00}$ & $95.64_{\pm 0.99}$ & & \\
        Baseline & BNB & $99.88_{\pm 0.18}$ & $\mathbf{98.36}_{\pm 0.86}$ & & \\
        V-scale & BNB & $\mathbf{100.00}_{\pm 0.00}$ & $95.76_{\pm 0.84}$ & & \\
        Baseline & GPTQ & $99.92_{\pm 0.11}$ & $94.64_{\pm 1.56}$ & & \\
        V-scale & GPTQ & $\mathbf{100.00}_{\pm 0.00}$ & $\mathbf{96.68}_{\pm 0.46}$ & & \\
        Baseline & AWQ & $99.92_{\pm 0.11}$ & $96.56_{\pm 1.24}$ & & \\
        V-scale & AWQ & $\mathbf{100.00}_{\pm 0.00}$ & $\mathbf{97.20}_{\pm 0.71}$ & & \\
        \midrule
        \rowcolor[gray]{0.9}\multicolumn{6}{l}{\emph{1B Models with YaRN Extension} (yarn\_factor$=4$)} \\
        Baseline & bfloat16 & $\mathbf{99.84}_{\pm 0.22}$ & $98.64_{\pm 0.26}$ & $98.20_{\pm 0.66}$ & $94.20_{\pm 0.75}$ \\
        V-scale & bfloat16 & $99.40_{\pm 0.40}$ & $\mathbf{99.56}_{\pm 0.38}$ & $\mathbf{98.32}_{\pm 0.41}$ & $\mathbf{99.04}_{\pm 0.33}$ \\
        Baseline & SmoothQuant & $\mathbf{99.72}_{\pm 0.23}$ & $97.48_{\pm 0.84}$ & $93.24_{\pm 1.21}$ & $88.40_{\pm 1.67}$ \\
        V-scale & SmoothQuant & $98.52_{\pm 0.30}$ & $\mathbf{97.96}_{\pm 0.17}$ & $\mathbf{93.72}_{\pm 0.63}$ & $\mathbf{96.16}_{\pm 0.70}$ \\
        Baseline & BNB & $99.60_{\pm 0.37}$ & $99.20_{\pm 0.37}$ & $98.12_{\pm 0.87}$ & $93.32_{\pm 0.77}$ \\
        V-scale & BNB & $\mathbf{99.76}_{\pm 0.17}$ & $\mathbf{99.96}_{\pm 0.09}$ & $\mathbf{99.40}_{\pm 0.35}$ & $\mathbf{98.52}_{\pm 0.50}$ \\
        Baseline & GPTQ & $\mathbf{99.80}_{\pm 0.20}$ & $97.32_{\pm 0.61}$ & $\mathbf{96.92}_{\pm 0.44}$ & $93.08_{\pm 1.18}$ \\
        V-scale & GPTQ & $99.64_{\pm 0.22}$ & $\mathbf{98.96}_{\pm 0.43}$ & $93.52_{\pm 1.14}$ & $\mathbf{97.72}_{\pm 0.27}$ \\
        Baseline & AWQ & $\mathbf{99.72}_{\pm 0.23}$ & $\mathbf{97.44}_{\pm 0.54}$ & $\mathbf{94.60}_{\pm 1.02}$ & $90.84_{\pm 1.07}$ \\
        V-scale & AWQ & $96.96_{\pm 0.77}$ & $94.56_{\pm 0.52}$ & $92.36_{\pm 1.40}$ & $\mathbf{98.32}_{\pm 0.84}$ \\
        \bottomrule
    \end{tabular}
\end{table}

\begin{table}[ht]
    \centering
    \caption{
    Needle-in-a-Haystack retrieval accuracy on the harder \texttt{niah\_single\_3} single-needle template for 1B models.
    Columns denote context length in tokens, and entries are mean accuracy percentages with standard deviations over $5$ random seeds; higher is better.
    The first block reports native-context evaluation up to 2048 tokens, while the second block uses YaRN extension with a factor of $4$ to evaluate up to 8192 tokens.
    On native 2048-token evaluation, V-scale improves over the baseline for \texttt{bfloat16}, SmoothQuant, BNB, and GPTQ, with AWQ as the exception.
    After YaRN extension, the result is mixed, with gains in several 2048-token and BNB settings but degradation under some 4096- and 8192-token quantized settings.
    Bold marks the better result within each baseline and V-scale pair under the same evaluation setting.
    }\label{tab:niah_single3_1b}
    \begin{tabular}{llcccc}
        \toprule
        Model & Quantization & 1024 & 2048 & 4096 & 8192 \\
        \midrule
        \rowcolor[gray]{0.9}\multicolumn{6}{l}{\emph{1B Models}} \\
        Baseline & bfloat16 & $\mathbf{99.68}_{\pm 0.11}$ & $85.48_{\pm 1.24}$ & & \\
        V-scale & bfloat16 & $98.60_{\pm 0.79}$ & $\mathbf{95.60}_{\pm 0.95}$ & & \\
        Baseline & SmoothQuant & $\mathbf{99.00}_{\pm 0.28}$ & $79.28_{\pm 0.93}$ & & \\
        V-scale & SmoothQuant & $96.00_{\pm 0.84}$ & $\mathbf{88.48}_{\pm 0.46}$ & & \\
        Baseline & BNB & $99.32_{\pm 0.18}$ & $79.80_{\pm 1.29}$ & & \\
        V-scale & BNB & $\mathbf{99.44}_{\pm 0.48}$ & $\mathbf{96.40}_{\pm 0.51}$ & & \\
        Baseline & GPTQ & $\mathbf{99.44}_{\pm 0.26}$ & $85.84_{\pm 1.88}$ & & \\
        V-scale & GPTQ & $98.84_{\pm 0.75}$ & $\mathbf{95.84}_{\pm 1.13}$ & & \\
        Baseline & AWQ & $\mathbf{99.44}_{\pm 0.26}$ & $\mathbf{96.16}_{\pm 1.25}$ & & \\
        V-scale & AWQ & $91.96_{\pm 1.01}$ & $85.20_{\pm 1.54}$ & & \\
        \midrule
        \rowcolor[gray]{0.9}\multicolumn{6}{l}{\emph{1B Models with YaRN Extension} (yarn\_factor$=4$)} \\
        Baseline & bfloat16 & $98.96_{\pm 0.26}$ & $91.56_{\pm 2.03}$ & $\mathbf{79.48}_{\pm 1.08}$ & $\mathbf{75.68}_{\pm 1.51}$ \\
        V-scale & bfloat16 & $\mathbf{99.08}_{\pm 0.70}$ & $\mathbf{94.36}_{\pm 1.01}$ & $77.16_{\pm 1.42}$ & $73.52_{\pm 1.10}$ \\
        Baseline & SmoothQuant & $97.00_{\pm 0.82}$ & $80.12_{\pm 2.11}$ & $\mathbf{71.96}_{\pm 2.07}$ & $64.00_{\pm 1.38}$ \\
        V-scale & SmoothQuant & $\mathbf{98.56}_{\pm 0.84}$ & $\mathbf{84.84}_{\pm 1.15}$ & $64.72_{\pm 1.57}$ & $\mathbf{66.92}_{\pm 1.20}$ \\
        Baseline & BNB & $98.60_{\pm 0.63}$ & $89.64_{\pm 1.11}$ & $75.68_{\pm 1.14}$ & $\mathbf{75.52}_{\pm 2.39}$ \\
        V-scale & BNB & $\mathbf{99.16}_{\pm 0.48}$ & $\mathbf{98.64}_{\pm 0.46}$ & $\mathbf{82.52}_{\pm 0.94}$ & $71.72_{\pm 1.98}$ \\
        Baseline & GPTQ & $98.92_{\pm 0.23}$ & $90.68_{\pm 1.30}$ & $\mathbf{75.72}_{\pm 1.32}$ & $\mathbf{73.68}_{\pm 0.70}$ \\
        V-scale & GPTQ & $\mathbf{99.36}_{\pm 0.57}$ & $\mathbf{94.36}_{\pm 0.99}$ & $61.04_{\pm 1.58}$ & $69.20_{\pm 1.97}$ \\
        Baseline & AWQ & $\mathbf{98.76}_{\pm 0.52}$ & $\mathbf{96.56}_{\pm 0.50}$ & $\mathbf{78.24}_{\pm 0.74}$ & $\mathbf{74.20}_{\pm 1.83}$ \\
        V-scale & AWQ & $98.24_{\pm 0.75}$ & $87.84_{\pm 1.11}$ & $49.80_{\pm 1.73}$ & $70.56_{\pm 1.98}$ \\
        \bottomrule
    \end{tabular}
\end{table}

\begin{table}[ht]
    \centering
    \caption{
    Needle-in-a-Haystack retrieval accuracy on the \texttt{niah\_multikey\_1} multi-key template for 1B models.
    This setting requires retrieving information when multiple key-value associations are present, making it more sensitive to long-context attention allocation than the saturated single-needle cases.
    Columns denote context length in tokens, and entries are mean accuracy percentages with standard deviations over $5$ random seeds; higher is better.
    Native-context models are evaluated at 1024 and 2048 tokens, and YaRN-extended models with a factor of $4$ are evaluated up to 8192 tokens.
    V-scale improves every listed precision/quantization and context-length combination, with especially large gains at 8192 tokens after YaRN extension.
    Bold marks the better result within each baseline and V-scale pair under the same evaluation setting.
    }\label{tab:niah_multikey1_1b}
    \begin{tabular}{llcccc}
        \toprule
        Model & Quantization & 1024 & 2048 & 4096 & 8192 \\
        \midrule
        \rowcolor[gray]{0.9}\multicolumn{6}{l}{\emph{1B Models}} \\
        Baseline & bfloat16 & $67.96_{\pm 1.34}$ & $65.04_{\pm 1.09}$ & & \\
        V-scale & bfloat16 & $\mathbf{73.88}_{\pm 1.06}$ & $\mathbf{71.64}_{\pm 3.46}$ & & \\
        Baseline & SmoothQuant & $64.44_{\pm 1.43}$ & $61.64_{\pm 2.91}$ & & \\
        V-scale & SmoothQuant & $\mathbf{71.64}_{\pm 1.44}$ & $\mathbf{69.60}_{\pm 2.19}$ & & \\
        Baseline & BNB & $66.16_{\pm 1.68}$ & $61.60_{\pm 1.83}$ & & \\
        V-scale & BNB & $\mathbf{74.04}_{\pm 1.65}$ & $\mathbf{70.56}_{\pm 2.19}$ & & \\
        Baseline & GPTQ & $66.20_{\pm 1.12}$ & $63.00_{\pm 1.54}$ & & \\
        V-scale & GPTQ & $\mathbf{67.88}_{\pm 1.06}$ & $\mathbf{65.56}_{\pm 2.25}$ & & \\
        Baseline & AWQ & $63.96_{\pm 2.10}$ & $61.68_{\pm 1.52}$ & & \\
        V-scale & AWQ & $\mathbf{66.84}_{\pm 1.75}$ & $\mathbf{67.76}_{\pm 3.10}$ & & \\
        \midrule
        \rowcolor[gray]{0.9}\multicolumn{6}{l}{\emph{1B Models with YaRN Extension} (yarn\_factor$=4$)} \\
        Baseline & bfloat16 & $57.76_{\pm 2.29}$ & $53.88_{\pm 1.42}$ & $48.00_{\pm 2.58}$ & $48.32_{\pm 1.85}$ \\
        V-scale & bfloat16 & $\mathbf{59.28}_{\pm 1.77}$ & $\mathbf{59.72}_{\pm 2.33}$ & $\mathbf{56.16}_{\pm 2.39}$ & $\mathbf{62.00}_{\pm 1.88}$ \\
        Baseline & SmoothQuant & $52.44_{\pm 1.84}$ & $49.64_{\pm 1.73}$ & $45.76_{\pm 2.17}$ & $48.96_{\pm 1.89}$ \\
        V-scale & SmoothQuant & $\mathbf{57.92}_{\pm 2.26}$ & $\mathbf{57.92}_{\pm 1.98}$ & $\mathbf{52.88}_{\pm 1.18}$ & $\mathbf{57.64}_{\pm 1.28}$ \\
        Baseline & BNB & $54.76_{\pm 2.60}$ & $46.28_{\pm 1.38}$ & $42.72_{\pm 2.96}$ & $47.56_{\pm 2.22}$ \\
        V-scale & BNB & $\mathbf{58.00}_{\pm 2.40}$ & $\mathbf{58.40}_{\pm 2.20}$ & $\mathbf{55.88}_{\pm 1.43}$ & $\mathbf{62.60}_{\pm 1.06}$ \\
        Baseline & GPTQ & $56.28_{\pm 1.64}$ & $50.24_{\pm 0.84}$ & $45.20_{\pm 2.81}$ & $46.72_{\pm 1.56}$ \\
        V-scale & GPTQ & $\mathbf{56.68}_{\pm 1.36}$ & $\mathbf{57.08}_{\pm 0.91}$ & $\mathbf{51.16}_{\pm 1.61}$ & $\mathbf{58.96}_{\pm 2.20}$ \\
        Baseline & AWQ & $54.04_{\pm 1.91}$ & $49.32_{\pm 2.02}$ & $42.52_{\pm 2.98}$ & $47.04_{\pm 2.04}$ \\
        V-scale & AWQ & $\mathbf{59.24}_{\pm 1.31}$ & $\mathbf{59.00}_{\pm 1.46}$ & $\mathbf{49.64}_{\pm 1.95}$ & $\mathbf{58.48}_{\pm 2.02}$ \\
        \bottomrule
    \end{tabular}
\end{table}

\subsection{Controlled activation quantization}\label{app:activation_quantization}

We evaluate the 1B baseline and V-scale models under three symmetric INT8 configurations to examine sensitivity to activation range. In the \emph{all-Linear dynamic-token} setting, every Linear module except \texttt{lm\_head} uses static per-channel INT8 weights and dynamic per-token INT8 input activations. The \emph{all-Linear static-tensor} setting keeps the same target modules and weight format, and replaces dynamic per-token activation scaling with a static per-tensor scale. The \emph{down-projection-only static-tensor} setting applies the same static weight and activation formats only to modules matching \texttt{mlp.down\_proj}.
Dynamic per-token scaling adapts to the activation range of each token. Static per-tensor scaling shares a calibrated range across tokens, exposing sensitivity to activation outliers. Restricting this quantization to the MLP down-projection tests the pathway associated with the strongest MA reduction in Figure~\ref{fig:vscale_as_ma_compare}. These settings serve as controlled diagnostics of sensitivity to a fixed activation range.

We use the same downstream task configurations as in Appendix~\ref{app:additional_downstream}. For the ten-task accuracy average, we use length-normalized accuracy for ARC-C, ARC-E, HellaSwag, OpenBookQA, PIQA, and SCIQ, and accuracy for BoolQ, COPA, LAMBADA, and WinoGrande. NIAH uses the native 2048-token context and evaluation seeds $0$--$4$, matching the corresponding experiments in Appendix~\ref{app:additional_downstream}. We report the mean and sample standard deviation across these seeds.

\begin{table}[ht]
    \centering
    \caption{
    Standard downstream results under controlled activation quantization for 1B models. Each metric compares the baseline (B) and V-scale (V). Mean accuracy is the unweighted average of the ten accuracy metrics in Table~\ref{tab:additional_downstream_1b}. Bold marks the better result within each pair.
    }\label{tab:activation_quantization_standard}
    \setlength\tabcolsep{5pt}
    \begin{tabular}{lcccccc}
        \toprule
        & \multicolumn{2}{c}{WikiText PPL$\downarrow$}
        & \multicolumn{2}{c}{LAMBADA PPL$\downarrow$}
        & \multicolumn{2}{c}{Mean acc. (\%)$\uparrow$} \\
        \cmidrule(lr){2-3}\cmidrule(lr){4-5}\cmidrule(lr){6-7}
        Setting & B & V & B & V & B & V \\
        \midrule
        \texttt{bfloat16} & 22.84 & \textbf{22.83} & \textbf{15.27} & 15.91 & 51.62 & \textbf{52.23} \\
        All Linear, dynamic-token W8A8 & 23.16 & \textbf{23.04} & \textbf{16.35} & 16.90 & 51.77 & \textbf{52.21} \\
        All Linear, static-tensor W8A8 & 33.55 & \textbf{25.46} & 46.85 & \textbf{23.65} & 48.54 & \textbf{50.55} \\
        \texttt{down\_proj} only, static-tensor W8A8 & 32.94 & \textbf{25.21} & 43.63 & \textbf{23.07} & 48.37 & \textbf{50.67} \\
        \bottomrule
    \end{tabular}
\end{table}

Table~\ref{tab:activation_quantization_standard} shows that both models remain close to their \texttt{bfloat16} performance under dynamic per-token scaling. Static per-tensor scaling yields a much larger separation. With all Linear modules quantized, V-scale has substantially lower perplexity and higher mean accuracy than the baseline. Quantizing only the MLP down-projection produces similar perplexities and accuracies, locating most of the observed sensitivity in this pathway.

\begin{table}[ht]
    \centering
    \caption{
    NIAH retrieval under controlled activation quantization at the native 2048-token context length. Single-2, Single-3, and Multi-key denote the same templates as Tables~\ref{tab:niah_single2_1b}--\ref{tab:niah_multikey1_1b}. Entries are mean accuracy percentages with sample standard deviations over five evaluation seeds. Higher is better, and bold marks the better result within each baseline (B) and V-scale (V) pair.
    }\label{tab:activation_quantization_niah}
    \setlength\tabcolsep{3.2pt}
    \begin{tabular}{lcccccc}
        \toprule
        & \multicolumn{2}{c}{Single-2}
        & \multicolumn{2}{c}{Single-3} & \multicolumn{2}{c}{Multi-key}\\
        \cmidrule(lr){2-3}\cmidrule(lr){4-5}\cmidrule(lr){6-7}
        Setting & B & V & B & V & B & V \\
        \midrule
        \texttt{bfloat16} & $\mathbf{98.00}_{\pm 0.69}$ & $97.80_{\pm 0.42}$ & $85.48_{\pm 1.24}$ & $\mathbf{95.60}_{\pm 0.95}$ & $65.04_{\pm 1.09}$ & $\mathbf{71.64}_{\pm 3.46}$ \\
        All / dynamic & $96.04_{\pm 1.11}$ & $\mathbf{96.28}_{\pm 0.54}$ & $77.00_{\pm 1.26}$ & $\mathbf{91.40}_{\pm 1.22}$ & $63.44_{\pm 1.98}$ & $\mathbf{66.88}_{\pm 2.17}$ \\
        All / static & $1.52_{\pm 0.39}$ & $\mathbf{79.40}_{\pm 2.48}$ & $0.00_{\pm 0.00}$ & $\mathbf{75.32}_{\pm 2.29}$ & $2.76_{\pm 0.86}$ & $\mathbf{56.04}_{\pm 2.08}$ \\
        \texttt{down\_proj} / static & $2.16_{\pm 0.36}$ & $\mathbf{79.28}_{\pm 1.87}$ & $0.04_{\pm 0.09}$ & $\mathbf{79.08}_{\pm 1.95}$ & $3.40_{\pm 1.09}$ & $\mathbf{56.84}_{\pm 2.67}$ \\
        \bottomrule
    \end{tabular}
\end{table}

The retrieval results in Table~\ref{tab:activation_quantization_niah} show the same dependence on activation scaling. Dynamic per-token scaling preserves retrieval in both models. With static per-tensor scaling, the baseline falls to near-zero accuracy on all three tasks, whereas V-scale retains substantial retrieval capability. Quantizing only \texttt{down\_proj} nearly reproduces the all-Linear result. The contrast supports lower sensitivity to a fixed activation range in the MLP pathway.

\clearpage

\section{Additional theory and proofs}\label{app:theory}

This appendix consolidates the supplementary theoretical material used in the main text. Throughout, we fix one layer and one attention head, and use the notation introduced in Section~\ref{sec:setup}.

\subsection{Exact backpropagation identities}\label{app:theory:exact}

\begin{lemma}[Softmax Jacobian identity]\label{lem:softmax-jacobian-app}
Fix a row $t$ and consider the causal softmax vector $a_t = \mathrm{softmax}(z_t)$ on the support $\{j \le t\}$. Then for any $j,s \le t$,
\[
    \frac{\partial a_{tj}}{\partial z_{ts}} = a_{tj} \bigl(\mathbf{1}_{j=s} - a_{ts}\bigr).
\]
\end{lemma}
\begin{proof}
By definition, we have $a_{tj} = \frac{e^{z_{tj}}}{\sum\nolimits_{\ell \le t} e^{z_{t\ell}}}$. Differentiating with respect to $z_{ts}$ gives
\[
    \frac{\partial a_{tj}}{\partial z_{ts}}
    = \frac{e^{z_{tj}} \mathbf{1}_{j=s} \left(\sum\nolimits_{\ell \le t} e^{z_{t\ell}}\right) - e^{z_{tj}} e^{z_{ts}}}{\left(\sum\nolimits_{\ell \le t} e^{z_{t\ell}}\right)^2}
    = a_{tj} \bigl(\mathbf{1}_{j=s} - a_{ts}\bigr).
\]
\end{proof}

\begin{lemma}[Sensitivity of $y_t$ to a single logit]\label{lem:yt-logit-app}
For fixed $t$ and $s \le t$, we have
\[
    \frac{\partial y_t}{\partial z_{ts}} = a_{ts} (v_s - y_t).
\]
\end{lemma}
\begin{proof}
Using $y_t = \sum\nolimits_{j \le t} a_{tj} v_j$ and Lemma~\ref{lem:softmax-jacobian-app}, we have
\[
    \frac{\partial y_t}{\partial z_{ts}}
    = \sum\nolimits_{j \le t} \frac{\partial a_{tj}}{\partial z_{ts}} v_j
    = \sum\nolimits_{j \le t} a_{tj} (\mathbf{1}_{j=s} - a_{ts}) v_j
    = a_{ts} v_s - a_{ts} \sum\nolimits_{j \le t} a_{tj} v_j
    = a_{ts}(v_s - y_t).
\]
\end{proof}

\begin{proof}[Proof of Proposition~\ref{prop:v_agg_main}]
For fixed $t$, the Jacobian of $y_t$ with respect to $v_s$ is $\frac{\partial y_t}{\partial v_s} = a_{ts} I \mathbb{I}_{s\le t}$. Applying the chain rule from $v_s \to y_t \to \mathcal{L}$ yields
\[
    \nabla_{v_s} \mathcal{L}
    = \sum\nolimits_{t=s}^{T-1} \left(\frac{\partial y_t}{\partial v_s}\right)^\top \nabla_{y_t} \mathcal{L}
    = \sum\nolimits_{t=s}^{T-1} (a_{ts} I)^\top g_t
    = \sum\nolimits_{t=s}^{T-1} a_{ts} g_t.
\]
\end{proof}

\begin{proposition}[Exact key-side gradient]\label{prop:k-exact-app}
Define
\[
    \beta_{ts} := \langle g_t, v_s - y_t \rangle.
\]
Then we have
\[
    \nabla_{k_s} \mathcal{L} = \sum\nolimits_{t=s}^{T-1} a_{ts} \beta_{ts} \frac{q_t}{\sqrt{d_{\mathrm{head}}}}.
\]
\end{proposition}
\begin{proof}
The key $k_s$ affects the logits $z_{ts} = \langle q_t, k_s \rangle / \sqrt{d_{\mathrm{head}}}$ for all $t \ge s$. The chain rule gives
\[
    \nabla_{k_s} \mathcal{L}
    = \sum\nolimits_{t=s}^{T-1} \left(\frac{\partial z_{ts}}{\partial k_s}\right)^\top \frac{\partial \mathcal{L}}{\partial z_{ts}}.
\]
Using $\frac{\partial z_{ts}}{\partial k_s} = \frac{q_t}{\sqrt{d_{\mathrm{head}}}}$ and Lemma~\ref{lem:yt-logit-app}, we have
\[
    \frac{\partial \mathcal{L}}{\partial z_{ts}}
    = \left\langle \nabla_{y_t} \mathcal{L}, \frac{\partial y_t}{\partial z_{ts}} \right\rangle
    = \langle g_t, a_{ts}(v_s - y_t) \rangle
    = a_{ts} \beta_{ts}.
\]
Substituting these two identities gives the claim.
\end{proof}

\begin{proposition}[Exact query-side gradient and first-token corollary]\label{prop:q-exact-app}
For the query pathway,
\[
    \nabla_{q_s} \mathcal{L} = \sum\nolimits_{j \le s} a_{sj} \beta_{sj} \frac{k_j}{\sqrt{d_{\mathrm{head}}}}.
\]
Moreover, under causal masking, we have $\nabla_{q_0}\mathcal{L} = 0$.
\end{proposition}
\begin{proof}
The query $q_s$ affects only the logits $\{z_{sj}\}_{j \le s}$. Hence
\[
    \nabla_{q_s} \mathcal{L}
    = \sum\nolimits_{j \le s} \left(\frac{\partial z_{sj}}{\partial q_s}\right)^\top \frac{\partial \mathcal{L}}{\partial z_{sj}}
    = \sum\nolimits_{j \le s} \frac{k_j}{\sqrt{d_{\mathrm{head}}}} \left\langle g_s, \frac{\partial y_s}{\partial z_{sj}} \right\rangle.
\]
Applying Lemma~\ref{lem:yt-logit-app} with row $s$ and column $j$ gives $\frac{\partial y_s}{\partial z_{sj}} = a_{sj}(v_j - y_s)$, so we have
\[
    \frac{\partial \mathcal{L}}{\partial z_{sj}} = a_{sj} \langle g_s, v_j - y_s \rangle = a_{sj} \beta_{sj}.
\]
This proves the general formula.
For $s=0$, causality implies $a_{00}=1$ and $y_0=v_0$. Thus we have $\beta_{00}=0$ and $\nabla_{q_0} \mathcal{L}=0$.
\end{proof}

\subsection{Value-path theorem and deterministic bounds}\label{app:theory:vside}

We prove the main value-path theorem used in Section~\ref{sec:theory}.

Fix the model parameters and condition on the attention matrix $A=(a_{tj})$, writing $\mathbb{E}_A[\cdot]=\mathbb{E}[\cdot\mid A]$. Then $a_{ts}$, $M_s$, and $S_s$ are fixed, while the assumptions below describe the remaining variability of the upstream gradients.

\begin{assumption}[Mean-plus-noise model]\label{assump:mean-noise-app}
For every position $t\ge s$, the upstream gradient admits the decomposition
\[
    g_t = \mu + \varepsilon_t\in\mathbb{R}^{d_{\mathrm{head}}},
\]
where $\mu=\mathbb{E}_A[g_t]$ is shared across these positions, so $\mathbb{E}_A[\varepsilon_t]=0$. The shared conditional mean is a modeling assumption on the coherent component of the upstream gradients.
\end{assumption}

\begin{assumption}[Bounded conditional noise moments]\label{assump:cov-app}
There exist $\sigma^2,\rho\ge0$ such that for all $t \neq t'$,
\[
\mathbb{E}_A\|\varepsilon_t\|_2^2 \le \sigma^2, \quad
\bigl|\mathbb{E}_A\langle\varepsilon_t,\varepsilon_{t'}\rangle\bigr| \le \rho.
\]
\end{assumption}

Assumption~\ref{assump:mean-noise-app} uses a mean-plus-noise description of stochastic gradients~\cite{mandt2017sgd,mccandlish2018noise}, with a shared conditional mean across positions. Assumption~\ref{assump:cov-app} permits correlated noise; $\mu$, $\sigma^2$, and $\rho$ may depend on $A$. Under the zero-mean condition, its bounds equivalently constrain covariance traces:
\[
\operatorname{Tr}\operatorname{Cov}_A(\varepsilon_t)
=\mathbb{E}_A\|\varepsilon_t\|_2^2,\quad    \operatorname{Tr}\operatorname{Cov}_A(\varepsilon_t,\varepsilon_{t'})
=\mathbb{E}_A\langle\varepsilon_t,\varepsilon_{t'}\rangle.
\]
Here $\operatorname{Cov}_A$ denotes conditional covariance given $A$. Both covariance matrices are $d_{\mathrm{head}}\times d_{\mathrm{head}}$. Independence across tokens or coordinates is not required.

\begin{proof}[Proof of Theorem~\ref{thm:v_second_moment_main}]
By Proposition~\ref{prop:v_agg_main} and Assumption~\ref{assump:mean-noise-app},
\[
    \nabla_{v_s} \mathcal{L} = \sum\nolimits_{t=s}^{T-1} a_{ts}(\mu + \varepsilon_t) = \left(\sum\nolimits_{t=s}^{T-1} a_{ts}\right)\mu + \sum\nolimits_{t=s}^{T-1} a_{ts}\varepsilon_t = M_s \mu + \sum\nolimits_{t=s}^{T-1} a_{ts}\varepsilon_t.
\]
Therefore,
\[\begin{aligned}
    \mathbb{E}_A\bigl[\|\nabla_{v_s} \mathcal{L}\|_2^2\bigr]
    &= \mathbb{E}_A\left[\left\|M_s\mu + \sum\nolimits_{t=s}^{T-1} a_{ts}\varepsilon_t\right\|_2^2\right] \\
    &= M_s^2\|\mu\|_2^2 + 2\,\mathbb{E}_A\left\langle M_s\mu, \sum\nolimits_{t=s}^{T-1} a_{ts}\varepsilon_t \right\rangle + \mathbb{E}_A\left\|\sum\nolimits_{t=s}^{T-1} a_{ts}\varepsilon_t\right\|_2^2.
\end{aligned}\]
Under the conditioning, $a_{ts}$, $M_s$, and $\mu$ are fixed. Together with $\mathbb{E}_A[\varepsilon_t]=0$, this makes the cross term vanish. Expanding the final term gives
\[
    \mathbb{E}_A\left\|\sum\nolimits_{t=s}^{T-1} a_{ts}\varepsilon_t\right\|_2^2
    = \sum\nolimits_{t=s}^{T-1}\sum\nolimits_{t'=s}^{T-1} a_{ts} a_{t's} \, \mathbb{E}_A\langle \varepsilon_t, \varepsilon_{t'} \rangle.
\]
Splitting diagonal and off-diagonal terms yields
\[
\mathbb{E}_A\big[\|\nabla_{v_s}\mathcal{L}\|_2^2\big]
=
M_s^2\|\mu\|_2^2
+
\sum\nolimits_{t=s}^{T-1} a_{ts}^2 \, \mathbb{E}_A \|\varepsilon_t\|_2^2
+
\sum\nolimits_{\substack{s\le t,t'\le T-1 \\ t\ne t'}} a_{ts} a_{t's} \, \mathbb{E}_A\langle \varepsilon_t, \varepsilon_{t'} \rangle.
\]
Finally, Assumption~\ref{assump:cov-app} implies
\[
    \sum\nolimits_{t=s}^{T-1} a_{ts}^2 \, \mathbb{E}_A \|\varepsilon_t\|_2^2 \le \sigma^2 \sum\nolimits_{t=s}^{T-1} a_{ts}^2 = \sigma^2 S_s,
\]
and
\[
    \left|\sum\nolimits_{\substack{s\le t,t'\le T-1 \\ t\neq t'}} a_{ts} a_{t's} \, \mathbb{E}_A\langle \varepsilon_t, \varepsilon_{t'} \rangle \right|
    \le \rho \sum\nolimits_{\substack{s\le t,t'\le T-1 \\ t\neq t'}} a_{ts} a_{t's}
    = \rho(M_s^2 - S_s).
\]
These estimates give the upper bound. The lower bound follows from the nonnegativity of $\mathbb{E}_A\|\sum_{t=s}^{T-1}a_{ts}\varepsilon_t\|_2^2$.
\end{proof}

\begin{corollary}[Deterministic value-side upper bounds]\label{cor:v-deterministic-app}
For any realization of $\{g_t\}_{t=s}^{T-1}$, we have
\[
    \|\nabla_{v_s} \mathcal{L}\|_2 \le \sum\nolimits_{t=s}^{T-1} a_{ts}\|g_t\|_2 \le M_s \max_{t\ge s} \|g_t\|_2,
\]
and also
\[
    \|\nabla_{v_s} \mathcal{L}\|_2 \le \sqrt{S_s} \left(\sum\nolimits_{t=s}^{T-1} \|g_t\|_2^2\right)^{1/2}.
\]
\end{corollary}

\begin{proof}
The first bound follows from Proposition~\ref{prop:v_agg_main} and the triangle inequality. The second follows from Cauchy-Schwarz:
\[
    \left\|\sum\nolimits_{t=s}^{T-1} a_{ts} g_t\right\|_2
    \le \left(\sum\nolimits_{t=s}^{T-1} a_{ts}^2\right)^{1/2} \left(\sum\nolimits_{t=s}^{T-1} \|g_t\|_2^2\right)^{1/2}.
\]
\end{proof}

\subsection{Deterministic bounds for key and query paths}\label{app:theory:kq}

The main theorem concerns the value path. This subsection provides supporting bounds for the key and query pathways and clarifies the asymmetry.

\begin{assumption}[Uniform boundedness for the key-side bound]\label{assump:k-bounded-app}
There exist constants $Q_{\max}, B_{\max} > 0$ such that for all $t$,
\[
    \|q_t\|_2 \le Q_{\max},
    \quad
    |\beta_{ts}| \le B_{\max}.
\]
\end{assumption}

\begin{theorem}[Deterministic key-side upper bound]\label{thm:k-bound-app}
Under Assumption~\ref{assump:k-bounded-app}, we have
\[
    \|\nabla_{k_s} \mathcal{L}\|_2 \le \frac{Q_{\max} B_{\max}}{\sqrt{d_{\mathrm{head}}}} M_s.
\]
\end{theorem}

\begin{proof}
Starting from Proposition~\ref{prop:k-exact-app} and applying the triangle inequality,
\[
    \|\nabla_{k_s} \mathcal{L}\|_2
    = \left\|\sum\nolimits_{t=s}^{T-1} a_{ts}\beta_{ts} \frac{q_t}{\sqrt{d_{\mathrm{head}}}}\right\|_2
    \le \sum\nolimits_{t=s}^{T-1} a_{ts}|\beta_{ts}|\frac{\|q_t\|_2}{\sqrt{d_{\mathrm{head}}}}.
\]
Using $a_{ts}\ge 0$, $|\beta_{ts}|\le B_{\max}$, and $\|q_t\|_2\le Q_{\max}$ gives the result.
\end{proof}

\begin{lemma}[Row-sum-zero of logits gradients]\label{lem:row-sum-zero-app}
For a fixed row $s$, we have
\[
    \sum\nolimits_{j \le s} \frac{\partial \mathcal{L}}{\partial z_{sj}} = 0.
\]
\end{lemma}

\begin{proof}
The derivation above gives $\frac{\partial \mathcal{L}}{\partial z_{sj}} = a_{sj}\langle g_s, v_j - y_s\rangle$. Therefore, we have
\[\begin{aligned}
    \sum\nolimits_{j \le s} \frac{\partial \mathcal{L}}{\partial z_{sj}}
    &= \left\langle g_s, \sum\nolimits_{j \le s} a_{sj}(v_j - y_s) \right\rangle \\
    &= \left\langle g_s, \sum\nolimits_{j \le s} a_{sj}v_j - y_s \sum\nolimits_{j \le s} a_{sj} \right\rangle \\
    &= \langle g_s, y_s - y_s \rangle = 0.
\end{aligned}\]
\end{proof}

\begin{assumption}[Bounded advantage and bounded key dispersion]\label{assump:q-bounded-app}
Let
\[
    \bar{k}_s := \sum\nolimits_{j \le s} a_{sj} k_j.
\]
There exist constants $B_{\max}^{(Q)}, \Delta K_{\max} > 0$ such that for all $j \le s$,
\[
    |\beta_{sj}| \le B_{\max}^{(Q)},
    \quad
    \|k_j - \bar{k}_s\|_2 \le \Delta K_{\max}.
\]
\end{assumption}

\begin{theorem}[Deterministic query-side upper bound in dispersion form]\label{thm:q-bound-app}
Under Assumption~\ref{assump:q-bounded-app},
\[
    \|\nabla_{q_s} \mathcal{L}\|_2 \le \frac{B_{\max}^{(Q)} \Delta K_{\max}}{\sqrt{d_{\mathrm{head}}}}.
\]
Moreover, under causal attention, $\nabla_{q_0}\mathcal{L} = 0$ exactly.
\end{theorem}

\begin{proof}
Using Lemma~\ref{lem:row-sum-zero-app}, for any vector $\bar{k}_s$ we have
\[
    \sum\nolimits_{j \le s} \frac{\partial \mathcal{L}}{\partial z_{sj}}\bar{k}_s = \bar{k}_s \sum\nolimits_{j \le s} \frac{\partial \mathcal{L}}{\partial z_{sj}} = 0.
\]
Hence the exact query-side gradient can be rewritten as
\[
    \nabla_{q_s} \mathcal{L}
    = \frac{1}{\sqrt{d_{\mathrm{head}}}} \sum\nolimits_{j \le s} \frac{\partial \mathcal{L}}{\partial z_{sj}} (k_j - \bar{k}_s)
    = \frac{1}{\sqrt{d_{\mathrm{head}}}} \sum\nolimits_{j \le s} a_{sj}\beta_{sj}(k_j - \bar{k}_s).
\]
Taking norms and applying the triangle inequality yields
\[
    \|\nabla_{q_s} \mathcal{L}\|_2
    \le \frac{1}{\sqrt{d_{\mathrm{head}}}} \sum\nolimits_{j \le s} a_{sj}|\beta_{sj}|\,\|k_j - \bar{k}_s\|_2.
\]
Using $\sum\nolimits_{j \le s} a_{sj}=1$, together with Assumption~\ref{assump:q-bounded-app}, proves the bound.
The statement for $s=0$ was already established in Proposition~\ref{prop:q-exact-app}.
\end{proof}

\subsection{RMSNorm as a gradient-compression mechanism}\label{app:theory:rmsnorm}

This subsection formalizes the RMSNorm mechanism used in Section~\ref{sec:theory} and measured empirically in Section~\ref{sec:ma_grad_reshape}.

For $x \in \mathbb{R}^d$, gain vector $\gamma \in \mathbb{R}^d$, and stabilizer $\epsilon_{\mathrm{rms}}>0$, define
\[
    \mathrm{rms}(x) := \sqrt{\frac{1}{d}\sum\nolimits_{i=1}^d x_i^2+\epsilon_{\mathrm{rms}}},
    \quad
    \mathrm{RMSNorm}(x) := \gamma \odot \frac{x}{\mathrm{rms}(x)}.
\]
Denote $r := \mathrm{rms}(x)$, $D := \mathrm{diag}(\gamma)$, and $y := \mathrm{RMSNorm}(x)$.

\begin{proposition}[Exact Jacobian of RMSNorm]\label{prop:rms-jac-app}
The Jacobian $J_{\mathrm{rms}}(x) := \partial y / \partial x \in \mathbb{R}^{d\times d}$ is
\[
    J_{\mathrm{rms}}(x) = \frac{1}{r}D - \frac{1}{r^3 d} D x x^\top.
\]
\end{proposition}

\begin{proof}
Since $y = D(x/r)$, we have
\[
    \frac{\partial y}{\partial x} = D\left(\frac{1}{r}I + x\,\frac{\partial}{\partial x}\left(\frac{1}{r}\right)\right).
\]
Then $\frac{\partial}{\partial x}\left(\frac{1}{r}\right) = -\frac{1}{r^2}\frac{\partial r}{\partial x} = -\frac{1}{r^3 d}x^\top$ gives the stated formula.
\end{proof}

Now we prove Theorem~\ref{thm:rms_compress_main} in Section~\ref{sec:theory}.

\begin{proof}[Proof of Theorem~\ref{thm:rms_compress_main}]
The chain rule gives $\nabla_x \mathcal{L} = J_{\mathrm{rms}}(x)^\top\nabla_y \mathcal{L}$. Proposition~\ref{prop:rms-jac-app} implies
\[
    J_{\mathrm{rms}}(x) = \frac{1}{r}D \left(I - \frac{x x^\top}{dr^2}\right).
\]
The matrix $I - \frac{x x^\top}{dr^2}$ has eigenvalues $\epsilon_{\mathrm{rms}}/r^2$ and $1$. Thus we have
\[
    \|J_{\mathrm{rms}}(x)\|_{\mathrm{op}}
    \le \frac{1}{r}\|D\|_{\mathrm{op}} = \frac{\|\gamma\|_\infty}{r}.
\]
\end{proof}

\begin{proposition}[Activation norm as a sufficient condition for bounded backpropagated gradients]\label{prop:required-rms-app}
Fix a target gradient scale $\tau > 0$. A sufficient condition for $\|\nabla_x \mathcal{L}\|_2 \le \tau$ is
\[
    \mathrm{rms}(x) \ge \frac{\|\gamma\|_\infty}{\tau} \|\nabla_y \mathcal{L}\|_2.
\]
Thus, for fixed $\tau$ and gain $\gamma$, the sufficient RMS level scales with the upstream gradient norm.
\end{proposition}

\begin{proof}
By Theorem~\ref{thm:rms_compress_main}, we have
\[
    \|\nabla_x \mathcal{L}\|_2 \le \frac{\|\gamma\|_\infty}{\mathrm{rms}(x)}\|\nabla_y \mathcal{L}\|_2.
\]
To make the right-hand side at most $\tau$, it suffices that $\mathrm{rms}(x) \ge \frac{\|\gamma\|_\infty}{\tau} \|\nabla_y \mathcal{L}\|_2$.
\end{proof}

Proposition~\ref{prop:required-rms-app} means that in a pre-norm Transformer, if a token receives unusually large gradient pressure, increasing its token-wise activation norm provides a direct channel for reducing the gradient magnitude passed to earlier layers. This is the mathematical sense in which massive activations can act as a gradient-compression response.

\subsection{V-scale Jacobian and gradient attenuation}\label{app:theory:vscale}

We first justify the initialization scale used in Section~\ref{sec:vscale}, then derive the V-scale Jacobian and its gradient attenuation properties. Recall that
\[
    C_{\ell,h}=d_{\mathrm{model}}d_{\mathrm{head}}\sigma^2\exp(\theta_{\ell,h}),
    \qquad \sigma=0.02, \quad \theta_{\ell,h}=0\ \text{at initialization}.
\]
RMSNorm with unit gain at initialization gives $\|\widetilde h\|_2^2\approx d_{\mathrm{model}}$. For zero-mean independent entries of $W_V$ with variance $\sigma^2$, the value state $v=W_V\widetilde h$ therefore satisfies
\[
    \mathbb{E}_{W_V}\!\left[\|v\|_2^2\mid\widetilde h\right]
    =d_{\mathrm{head}}\sigma^2\|\widetilde h\|_2^2
    \approx d_{\mathrm{model}}d_{\mathrm{head}}\sigma^2.
\]
Thus, $C_{\ell,h}$ starts at the typical squared value-state norm, while the learned factor $\exp(\theta_{\ell,h})$ allows this scale to adapt across layers and heads.

We now analyze the V-scale map used in Section~\ref{sec:vscale}. On the pre-aggregation value state $v$, define
\[
    \hat{v} = \phi(r) v,
    \quad
    r := \|v\|_2^2,
    \quad
    \phi(r) := \frac{r}{r+C},
    \quad
    C>0.
\]

\begin{proposition}[Exact Jacobian of V-scale]\label{prop:vscale-jac-app}
The Jacobian $J_\phi(v) := \partial \hat{v} / \partial v$ is
\begin{equation}\label{eq:vscale_jac}
    J_\phi(v) = \phi(r) I + \frac{2C}{(r+C)^2} vv^\top.
\end{equation}
\end{proposition}

\begin{proof}
Since $\hat{v} = \phi(r)v$ and $r = v^\top v$, we have
\[
    \frac{\partial \hat{v}}{\partial v}
    = \phi(r)I + v\left(\frac{\partial \phi(r)}{\partial v}\right)^\top.
\]
Since $\phi'(r) = \frac{C}{(r+C)^2}$ and $\frac{\partial r}{\partial v} = 2v$, 
substituting
\[
    \frac{\partial \phi(r)}{\partial v} = \phi'(r)\frac{\partial r}{\partial v} = \frac{2C}{(r+C)^2}v
\]
yields the claim.
\end{proof}

\begin{proposition}[Jacobian spectrum of V-scale]\label{prop:vscale_main}
Let $r=\|v\|_2^2$.
Then the Jacobian in~\eqref{eq:vscale_jac} has eigenvalue $\lambda_\perp(r)=\frac{r}{r+C}$ on the $(d_{\mathrm{head}}-1)$-dimensional subspace orthogonal to $v$, and eigenvalue $\lambda_\parallel(r) = \frac{r^2+3Cr}{(r+C)^2}$ along the radial direction $v$ itself.
Moreover, $\max_{r\ge 0} \lambda_\parallel(r)={9}/{8}$, attained at $r=3C$.
\end{proposition}

\begin{proof}
Because $vv^\top$ is rank one, for any $u\perp v$ we have $vv^\top u = 0$, so Proposition~\ref{prop:vscale-jac-app} gives $J_\phi(v)u = \phi(r)u$. Along the radial direction,
\[
    J_\phi(v)v = \phi(r)v + \frac{2C}{(r+C)^2}vv^\top v = \left(\phi(r) + \frac{2Cr}{(r+C)^2}\right)v.
\]
This proves the eigenvalue formulas.

For the maximum, define $f(r) := \lambda_\parallel(r) = \frac{r^2+3Cr}{(r+C)^2}$.
Direct differentiation shows that $f'(r)=0$ at $r=3C$, and substituting this value gives $f(3C)=9/8$.
\end{proof}

\begin{corollary}[V-scale is a value-path gradient valve]\label{cor:vscale_main}
Let $g_{\hat v}=\nabla_{\hat v}\mathcal{L}$ be the upstream gradient at the scaled value state.
Then we have $\nabla_v \mathcal{L} = J_{\phi}(v)^\top g_{\hat v} = J_{\phi}(v) g_{\hat v}$,
and therefore $\|\nabla_v \mathcal{L}\|_2 \le \lambda_\parallel(\|v\|_2^2)\|g_{\hat v}\|_2$.
In particular, if $\|v\|_2^2\ll C$, then the value-path backward signal is strongly attenuated as
\[
\|\nabla_v \mathcal{L}\|_2
\lesssim
\frac{3\|v\|_2^2}{C}
\,\|g_{\hat v}\|_2.
\]
\end{corollary}

\begin{proof}
The first identity is the chain rule. Since $J_\phi(v)$ is symmetric, its operator norm equals its largest eigenvalue, which is $\lambda_\parallel(r)$ by Proposition~\ref{prop:vscale_main}. The small-$r$ expansion follows directly from the expression for $\lambda_\parallel(r)$.
\end{proof}

These results show that V-scale introduces an additional value-path gradient valve. In the regime where sink value states have relatively small norm, the attenuation factor in Corollary~\ref{cor:vscale_main} can be substantially below one, which is precisely the regime relevant to the observed intervention effect.

\end{document}